\documentclass{amsart}

\usepackage[hidelinks]{hyperref}
\usepackage{comment}
\usepackage{aageneral}
\usepackage{aamath}
\usepackage{mhequ}
\usepackage{fullpage}

\usepackage{booktabs}       
\usepackage{nicefrac}       
\usepackage{microtype}      

\newtheorem{assumption}{Assumption}

\usepackage{stmaryrd}

\def\supp{\mathrm{supp}}

\mathtoolsset{showonlyrefs=true}

\usepackage{todonotes}

\def\din{d}

\def\supp{\mathrm{supp}}
\def\span{\mathrm{span}}

\def\whh{W_{hh}}
\def\wxh{W_{xh}}
\def\why{W_{hy}}

\def\wwhh{\mathbf W_{hh}}
\def\wwxh{\mathbf W_{xh}}
\def\wwhy{\mathbf W_{hy}}

\def\WWhh{\mathbf W_{hh}}
\def\WWxh{\mathbf W_{xh}}
\def\WWhy{\mathbf W_{hy}}

\def\Phh{P_{hh}}
\def\Pxh{\mathbb R^\din}

\def\pp{P}
\def\phh{P_{hh}}

\def\dhh{D_{hh}}
\def\dxh{D_{xh}}
\def\dhy{D_{hy}}
\def\dhh{D_{hh}}
\def\dxh{D_{xh}}
\def\dhy{D_{hy}}
\def\ds{D_{\sigma}}

\def\ghh{G_{hh}}

\def\fhh{M_{hh}}
\def\fxh{M_{xh}}

\def\dehh{\Delta_{hh}}

\def\HH{H_\sigma}

\def\hh{\mathbf H}
\def\Hhh{H_{hh}}
\def\Hxh{H_{xh}}

\def\hhhh{\mathbf H_{hh}}
\def\hhxh{\mathbf H_{xh}}

\def\Ohh{\Omega_{h}}

\def\axh{a_{xh}}

\def\ahf{a_{h1}}
\def\ahs{a_{h2}}

\usepackage{bbm}
\def\sign{\text{sign}}

\def\xx{\mathbf x}

\def\DD{\mathcal D}
\def\ww{\mathbf W}
\def\WW{\mathbf W}

\def \TT{\mathcal T}

\renewcommand\aa[1]{#1}

\newcommand\nx[1]{\mathbb E_X\left[|#1|\right]^2}
\newcommand\nxx[1]{\mathbb E_X\left[|#1|^2\right]}

\makeatletter
\newtheorem*{rep@theorem}{\rep@title}
\newcommand{\newreptheorem}[2]{%
\newenvironment{rep#1}[1]{%
 \def\rep@title{#2 \ref{##1}}%
 \begin{rep@theorem}}%
 {\end{rep@theorem}}}
\makeatother

\newreptheorem{theorem}{Theorem}
\newreptheorem{lemma}{Lemma}
\newreptheorem{proposition}{Proposition}

\let\theta=\vartheta
\title{Global optimality of Elman-type RNNs  in the mean-field regime}

\author{ Andrea Agazzi$^{1,2}$ and Jianfeng Lu$^{2,3,4}$ and  Sayan Mukherjee$^{2,5,6,7,8,9}$}
\address[1]{Department of Mathematics,
  Università di Pisa,
  Pisa, IT}

\address[2]{Department of Mathematics,
Duke University,
Durham, NC 27708}
\address[3]{Department of  Physics,
Duke University,
Durham, NC 27708}
\address[4]{Department of Chemistry,
Duke University,
Durham, NC 27708}
\address[5]{Center for Scalable Data Analytics and Artificial Intelligence,
Universit\"at Leipzig,
Leipzig, DE\\}
\address[6]{Max Planck Institute for Mathematics in the Sciences,
Leipzig, DE\\}
\address[7]{
Department of Statistical Science,
Duke University,
Durham, NC 27708}
\address[8]{
Department of Computer Science,
Duke University,
Durham, NC 27708}
\address[9]{
Department of Biostatistics \& Bioinformatics,
Duke University,
Durham, NC 27708 \newline }

\email{ andrea.agazzi@unipi.it, jianfeng@math.duke.edu, sayan.mukherjee@mis.mpg.de}

\date{\today}
\usepackage[sort,compress,numbers]{natbib}

\begin{document}

\begin{abstract}
{We analyze Elman-type Recurrent Reural Networks (RNNs) and their training in the mean-field regime.} Specifically, we show convergence of {gradient descent training dynamics} of the RNN to the corresponding mean-field formulation in the large width limit. We also show that the fixed points of the limiting infinite-width dynamics are globally optimal, under some assumptions on the initialization of the weights. Our results establish optimality for feature-learning with wide RNNs in the mean-field regime.
\end{abstract}

\maketitle

\section{Introduction}

During the last decade, artificial intelligence and in particular deep leaning have achieved a significant series of groundbreaking successes, partly due to the unprecedented increase of data and computational power at our disposal. Notably, the range of disciplines that have recently been revolutionized by machine learning is virtually unlimited: from medicine \cite{mlinm} to finance \cite{mlfinance}, from games \cite{Silver:16,vinyals19} to image analysis \cite{lecun89}, to the point where almost no domain has remained unaltered by the emergence of these technologies.

This revolution would have been unthinkable without the advent of deep neural networks. This extremely flexible family of function approximators has outperformed classical methods in almost every domain where it has been applied.
A discipline that has been profoundly revolutionized by these models is the analysis of time series and, more generally, the problem of learning dynamical systems. For instance, Recurrent Neural Networks (RNNs) \cite{jordan1, rumelhart85} and more specifically Long-Short Term Memory (LSTM) RNNs \cite{schmidhuber1, schmidhuber2} and Gated Recurrent Units \cite{GRUs1} have dramatically increased the predictive performance of machine learning in this context. These models take as input temporal sequences of data and act iteratively on the elements of such sequences, storing the information about previous timepoints into the hidden state of the network. This structure allows to learn datasets with strong time-correlations using relatively few parameters, and has provided benchmarks for state-of-the-art time-series learning algorithms for over a decade.
However, despite the groundbreaking success of these models in practice, the theoretical underpinnings of such success remain elusive to the computer science community. More specifically, many questions about the theoretical reasons for the performance of these models applied far into the overparametrized regime, such as for example explanations for their optimal behavior and their generalization error, remain open.

Only recently, a theory of neural
network learning has started to emerge in the context of wide, single-layer neural networks. The two main theoretical frameworks are based on either understanding mean-field training dynamics
\cite{ChizatBach18, RotVE18, MeiMonNgu18, wojtowytsch20, spiliopoulos18, agazzi20, ChizatBach20} or based on linearized dynamics in the overparametrized regime \cite{hongler18, ChizatBach182,Montanari:19b}.
These two frameworks provide contrasting explanations for the success of neural networks. \aa{On one hand the linearized dynamics gives strong convergence guarantees for the training process but fails to explain the feature-learning properties of neural networks. On the other hand the the mean-field framework is better at  accurately
capturing the
highly nonlinear dynamics arguably resulting in feature-learning \cite{Montanari:19c}, but the resulting training process is typically harder to analyze. Consequently, it remains a challenge to extend the mean-field results listed above to more realistic structures such as RNNs.}

This paper aims to extend the mean-field framework to Elman-type RNNs. Specifically, we aim to establish optimality of the fixed points of the training dynamics for wide RNNs trained with classical gradient descent. This provides an explanation of the outstanding performance of these models, in a certain idealized regime.

\subsection{Previous results}

The training dynamics of neural networks in the mean-field infinite width limit was pioneered by the series of papers \cite{MeiMonNgu18, ChizatBach18, RotVE18,spiliopoulos18}. Here, the authors proved that the training dynamics of infinitely wide, single layer neural networks in the mean-field regime can be studied by representing the parametric state of the network as a probability distribution in the space of weights. Using this representation it was possible to prove that
the limiting points of the training dynamics are global optimizers of the loss function.\\
These ideas have been extended to the reinforcement learning setting \cite{spiliopoulos22,agazzi20,agazzi19}, to non-differeniable network nonlinearities such as ReLU \cite{wojtowytsch20} and to the deep ResNet architecture \cite{lu20}. A recent series of papers \cite{pham20global,pham21global} has bypassed the difficulties related to the representation of the network state as a distribution by introducing the \emph{neuronal embedding} framework. This framework allows for the investigation of the mean-field dynamics of deep, purely feedforward neural networks. However, none of these results can be directly applied to the RNN setting: \aa{the presence of weight-sharing in  the RNN structure and the interaction of the unrolled network with the input violate fundamental assumptions in these analyses. }

The performance of RNNs in the infinite width limit was studied in \cite{rntk20}. Here, the authors explored the performance of the network in the so-called Neural Tangent Kernel (NTK) regime, arising under a particular scaling of the weights at initialization. This scaling linearizes the training dynamics of the network, which behaves essentially like a kernel method. It is therefore widely believed that in this regime feature learning is not possible. A recent paper \cite{yang20} considers a similar scaling to \cite{rntk20} in combination with more general architectures. In contrast to these works, the mean-field scaling we consider retains the nonlinear training dynamics.

Finally, a seemingly related result about mean-field theory for RNNs has been presented in \cite{chen18}. That work, however, uses \emph{dynamical} mean-field theory to explain the role of gating in RNN architectures and thus our proof techniques differ greatly from that paper. \aa{The scope of the results is also significantly different, as their results aim to explore forward propagation of signal
through vanilla RNNs, and do not aim to establish optimality of the fixed points after training.}

\subsection{Contributions}

This paper adapts the neuronal embedding analysis framework developed in \cite{pham20global,pham21global} to unrolled Elman-type RNNs \cite{Elman90}. We prove optimality of the fixed points of the training
dynamics in the mean-field regime under some  assumptions \aa{on} the expressive power of the
network at initialization. Specifically we prove:
\begin{enumerate}
\item Convergence of the dynamics of the finite-width RNN to its infinite-width limit. To do so we adapt the coupling formulation presented in \cite{pham20global} in the context of fully-connected feedforward networks to the RNN framework, thereby extending it to networks with weight-sharing.
\item Gradient descent trains these networks to optimal fixed points given infinite training time. This optimality result holds in the feature-learning regime, as opposed to previous results that hold in the NTK regime.
\item \aa{To prove the above results, we show universal approximation for deep neural networks with uniformly bounded hidden weights. This result extends classical universal approximation theorems, where weights are critically assumed to be in a vector space and, as such, to be unbounded.}
\end{enumerate}
A standard initialization assumption in feedforward neural networks, for example $3$-layer networks, with a large number of nodes is to initialize the weights randomly and independently. In this paper, we further observe that feedback in an RNN requires stronger assumptions on the weights of the network at initialization to achieve a comparable level of expressivity as a $3$-layer feedforward network. We examine this issue in some detail through our analysis.


The paper is organized as follows: in Section~\ref{s:notation} we introduce the notation and the model being investigated, together with its mean-field limit. Then, in Section~\ref{s:results} we outline our main results. The results are exemplified with some numerical experiments in Section~\ref{s:numerics}, and conclusions follow in Section~\ref{s:conclusion}. The proofs of our main theorems are given in the appendix.

\section{Notation}\label{s:notation}

\subsection{Predictors}\label{s:predictors}
To put the data-generation process in an abstract framework for dynamical systems, we consider as predictors subsets of a bi-infinite observation sequence $\mathbf{x} \in  (\mathbb R^\din)^{\mathbb Z}$. For a given subshift $\TT~:~(\mathbb R^\din)^{\mathbb Z} \to (\mathbb R^\din)^{\mathbb Z}$, we generate the elements of $\xx$ as $\mathbf x_{k+1} = \TT(\mathbf x)_k$.
 We make the following assumption on the underlying dynamical system
\begin{assumption}\label{a:01a}
  There exists a continuous function $T~:~\mathbb R^\din \to \mathbb R^\din$ such that $\mathbf x_{k+1} = T(\mathbf x_k)$  for all $k \in \mathbb Z$. We further assume that this map is uniquely ergodic \aa{(upon possibly restricting it to a forward invariant set $\mathbb X$)} and that the corresponding invariant measure has finite fourth moments. For the definition of unique ergodicity, see \cite{dstheory}.
\end{assumption}


We denote by $\Pi^0~:~(\mathbb R^\din)^{\mathbb Z}\to \mathbb R^\din$ the projection of a bi-infinite sequence on its $0$-th element. We further define $\nu \in \mathcal M_+^1(\mathbb R^{\mathbb Z})$ as the invariant measure of the map $\mathcal T$, which exists and is unique by the above assumption. The marginal $\nu_0\in \mathcal M_+^1(\mathbb R)$ on the $0$-th component of $\nu$  is the invariant measure of $T$, and $\Pi_\#^0 \nu = \nu_0$.

\subsection{Loss function}
We assume that we have access to an infinite-length sample from the invariant measure $\nu$, from a dynamical system satisfying  \aref{a:01a} to train the RNN.
Our objective is to learn a map $F^*~:~(\mathbb R^\din)^{\mathbb N} \to \mathbb R$  from sequences of arbitrary length to reals. We restrict our attention to functions with a fixed, finite memory $L \in \mathbb N$.
\begin{assumption}\label{a:02}
The function $F^*$ only depends on $\{\mathbf x_{-L}, \dots, \mathbf x_0\}$, for a fixed $L \in \mathbb N$.
\end{assumption}
Our objective is to learn an estimate of $F^*$
by minimizing the
sample Mean Squared Error (MSE) between the target function $F^*$ and a parametric family of estimators $\{F(\,\cdot\,;W)\}_W$ indexed by the parameter vector $W$ on an observation sequence of length $K$. In other words, we aim to find
the minimizer $\hat F \in \{F(\,\cdot\,;W)\}_W$ of the empirical risk
\begin{equ}
\mathcal L_K(F^*, \hat F) := \frac 1 K \sum_{k = 1}^K \frac 12(F^*(\TT^k(\mathbf x)) - \hat F(\TT^k(\mathbf x))^2 .
\end{equ}
In the large sample limit $K \to \infty$, the above loss function can be rewritten as the population risk
\begin{equs}\label{e:mse1}
\mathcal L(W)  = \lim_{K \to \infty} \mathcal L_K(F^*(\,\cdot\,), \hat F(\,\cdot\,;W)) =   \frac 12 \int (F^*(\mathbf x) - \hat F(\mathbf x;W))^2 \nu(\d \mathbf x)\,,
\end{equs}
expressed above as a function of the parameters of the estimator. While our analysis extends to more general loss functions, for concreteness and ease of exposition we restrict our discussion to the MSE.

\subsection{RNN structure} The family of models we consider are Elman-type Recurrent Neural Networks  of hidden width $n \in \mathbb N$. Such a neural network can be written as
\begin{equ}\label{e:finitewidth}
\begin{aligned}
\hat F(\mathbf x;\mathbf W) & = 
\mathbf H_{hy}({\mathbf x}) \\
\mathbf H_{hy}({\mathbf x}) & = { \frac 1 n} \wwhy \sigma_h(\mathbf H_{hh}(\mathbf x,0) + \mathbf H_{xh}(\mathbf x_0))\\
\mathbf H_{hh}(\mathbf x, k) & = { \frac 1 n}  \wwhh \sigma_h(\mathbf H_{hh}(\mathbf x,k+1) + \mathbf H_{xh}(\mathbf x_{-(k+1)}))\\
\mathbf H_{hh}(\mathbf x, L) & = 0\\
\mathbf H_{xh}(\mathbf x_k) & = \wwxh  \cdot \mathbf x_{k}
\end{aligned}
\end{equ}
where we assume that $\mathbf x_k \in \mathbb R^\din$, $\wwxh\in \mathbb R^{n \times \din}$, $\wwhy\in \mathbb R^{ n}$, $\wwhh\in \mathbb R^{n \times n}$ and activation function $\sigma_h~:~\mathbb R \to \mathbb R$
is applied component-wise.   The structure of the network is represented in \fref{f:structure}.

\begin{figure}
\centering
\includegraphics[width = 0.45\linewidth]{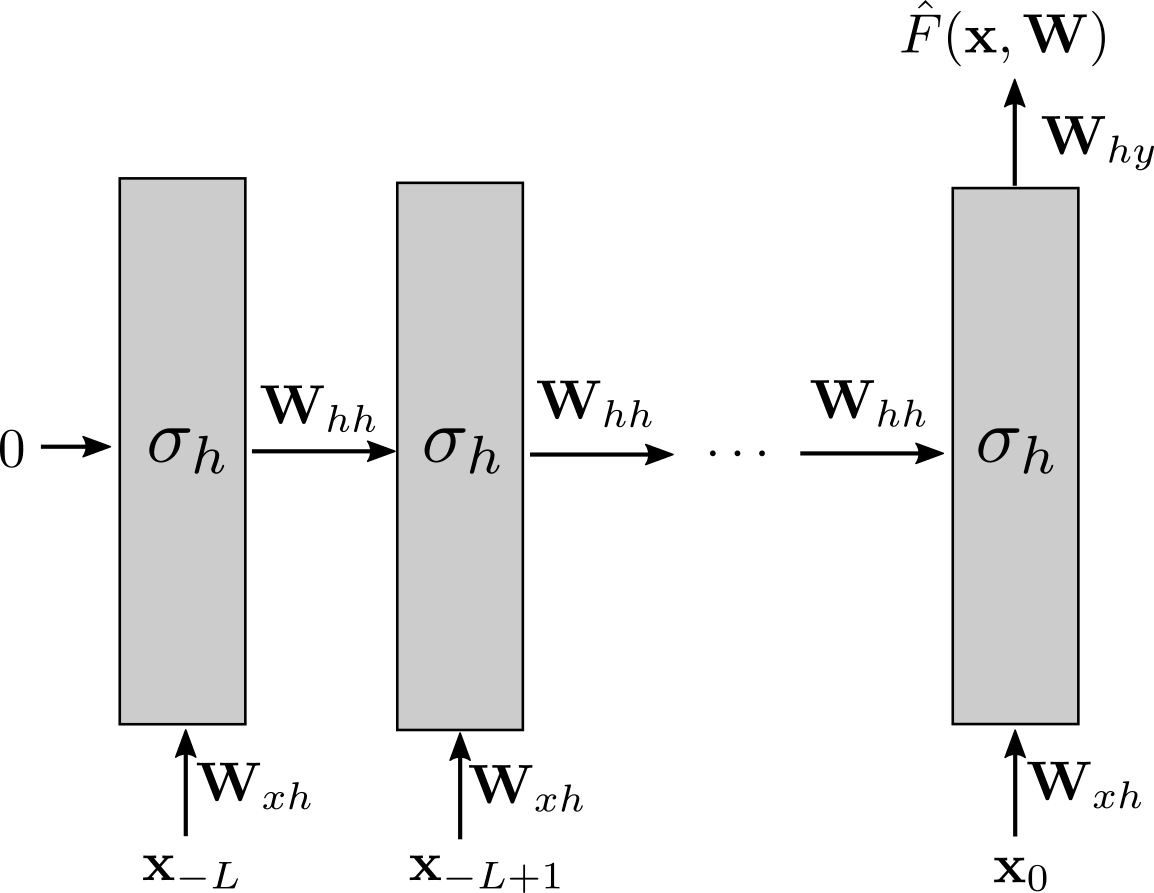}
\caption{The many-to-one structure of the Elman-type RNN. }
\label{f:structure}
\end{figure}%

To investigate the convergence properties of RNNs as $n\to \infty$, we will apply the neuronal embedding formalism from \cite{pham20global,pham21global}.
This formalism lifts the labeling of the neurons of the network to an abstract probability space $(\Ohh, \mathcal F_{h}, \phh)$, and the neural network weights are interpreted as a function of these abstract indices. This lifting allows for the representation of
any network as a specific choice of labelings, and
 equivalent relabelings of the neural network weights are different realizations of an abstract (random) labeling process.
In this formalism, the weight functions can then be written as
\begin{equ}
\wxh(\theta)\in \mathbb R^d
\qquad \whh(\theta,\theta')\in \mathbb R
\qquad \why(\theta')\in \mathbb R
\end{equ}
for $\theta, \theta' \in \Ohh$. A precise definition of $\theta, \theta'$ and of the coupling procedure to identify the neuronal embedding with the infinite width limit of the network \eref{e:finitewidth} is given in the next section. The mean-field representation is the  continuous version of the network introduced above, representing matrix multiplications as integral kernels and can be written as
\begin{equ}\label{e:mfrnn}
\begin{aligned}
\hat F(\mathbf x;W)  = &
H_{hy}(\mathbf x) \\
H_{hy}(\mathbf x)  = &\int \why(\theta) \sigma_h(H_{hh}(\theta; \mathbf x,0)+H_{xh}(\theta; \mathbf x_0)) \phh(\d\theta)\\
H_{hh}(\theta; \mathbf x,k)  =& \int \whh(\theta, \theta')
\sigma_h(H_{hh}(\theta';\mathbf x,k+1)+ H_{xh}(\theta';\mathbf x_{-(k+1)})) \phh(\d \theta')\\
H_{hh}(\theta;\mathbf x, L)  \equiv & \, 0\\
H_{xh}(\theta; \mathbf x_k)  = & \wxh(\theta) \cdot \mathbf x_{k}
\end{aligned}
\end{equ}

As the next example shows, any finite-width RNN $\hat F(\ww;\xx)$ can be \emph{embedded} into the mean-field representation.
\begin{example}(Finite-width RNN)\label{ex:1} For any choice of parameters $\ww = \{\wwxh, \wwhh, \wwhy\}$ for a width-$n$ network for $n<\infty$ and assuming $\din = 1$ we can set
  $
  \Ohh  := \{1,2,\dots, n\}\,.
$
The measure $\Phh$ can be chosen as the uniform measure on $\{1,2,\dots, n\}$. Then, it is readily seen that, setting $\whh(i,j) := (\wwhh)_{ij}$, $\wxh(i) := (\wwxh)_{i}$ and $\why(j) := (\wwhy)_{j}$ for $i,j \in \{1,\dots, n\}$ we have that \eref{e:mfrnn} gives the same output as \eref{e:finitewidth}.
\end{example}

\subsection{Initialization and Coupling procedure} We now introduce the coupling procedure that connects the evolution of finite-width neural networks \eref{e:finitewidth} to their mean-field representation \eref{e:mfrnn}. This coupling procedure is performed \emph{at initialization}, \ie before training starts. We will respectively denote the  weights of the finite-width network and of the mean-field limit at initialization by $  \mathbf W^0$ and $W^0$.
Instrumental to introducing the coupling procedure between the finite-width and the infinite-width neural network is the notion of \emph{neuronal embedding}. Given a family $I$ of initialization laws {indexed by the width $n$ of the hidden layer,}
\begin{equ}
  I = \{\rho_n~:~\rho_n \text{ is the law of $\mathbf W^0$ for a network of width $n$}\}
\end{equ}
we consider the parameters  $\mathbf W^0$
of the width-$n$ network as samples from the corresponding distribution $\rho_n \in I$.

We call $(\Ohh, \phh, W)$ a neuronal embedding for the neural network with initialization laws in $I$ if for every $\rho_n$ there exists a sampling rule $\bar P_{n}$ such that
\begin{enumerate}
    \item $\bar P_{n}$ is a distribution on $\Ohh^n$ (not necessarily a product distribution) with marginals given by $\phh$
    \item The mean-field weights $W = (\wxh, \whh, \why)$ are such that, if $(\theta(j))_j \sim \bar P_n$,  then for every $n$ with $i,j \in \{1,\dots,n\}$:
    \begin{equ}
      \text{Law}(\wxh(\theta(i)), \whh(\theta(i), \theta(j)), \why(\theta(j)))=\rho_n.
    \end{equ}
\end{enumerate}
The above definition decomposes the concept of neural network weights to two parts: the first part is a deterministic function of possibly continuous arguments and the second part consists of a random map $\theta$ transforming the index $i$ to a (random) argument of the weight function $W$. A finite-width network is then seen as a choice of the map $\theta$ and weight function $W$. The evolution of the weights is captured, for a choice of $\theta$, by the dependence of $W$ in time (the time evolution will be detailed in the next section). Specifically, we couple $\mathbf W^0$ and $W^0$ as follows:
\begin{enumerate}
    \item Given a family of initialization laws $I$, we choose $(\Ohh, \phh, W^0)$ to be a neuronal embedding of $I$ and initialize the dynamical quantities $W^0(\cdot)$.
    \item Given $n \in \mathbb N$ and the sampling rule $\bar P_n$, we sample $(\theta(1),\dots, \theta(n))\sim \bar P_n $ and set $\wwhh^0(i,j) = \whh^0(\theta(i), \theta(j))$, $\wwxh^0(i) = \wxh^0(\theta(i))$ and $\wwhy^0(j) = \why^0(\theta(j))$ for $\,j\in \{1,\dots , n\}$.
\end{enumerate}

  The key property of the neuronal embedding construction is the decomposition of the probability space generating an instance of the neural network into a product space over different layers. This decomposition captures the symmetry of the neural network's output under certain permutations of the indices of the neurons, thereby generalizing the representation as an empirical measure used in \cite{ChizatBach18, spiliopoulos18, RotVE18, MeiMonNgu18}. The following example helps clarify this analogy.

\begin{example}
In the case of the finite-width network discussed in Example~\ref{ex:1}, the sampling rule $\theta(i) = i + \omega$ with $\omega \in \Ohh$ common to the whole layer and distributed uniformly on $\Ohh$ satisfies the above conditions. We further notice that $\theta(i) = \iota(i)$ for any (random) permutation $\iota$ of $\{1,\dots, n\}$ realizes the same neural network, \ie a neural network with the \emph{same} weights $\ww$, up to permutation of the indices of its neurons.
\end{example}
While the example above illustrates the connection between neuronal embeddings and finite-width neural networks by using finite probability spaces, the same connection can be established more abstractly in the case of \abbr{iid} initializations for arbitrary and infinite-width networks by means of the Kolmogorov extension theorem.

\begin{example}
In the case of \abbr{iid} initialization, the neuronal embedding acquires a more explicit formulation. For a given probability space $(\Lambda, \mathcal G, P_0)$ we define $p_{xh}(c), p_{hh}(c,c')$ and $p_{hy}(c)$ which are respectively $\mathbb R^\din$-valued,  $\mathbb R$-valued and $\mathbb R$-valued random processes indexed by $(c,c') \in [0,1]\times [0,1]$. For any $n$ and any collection of indices $\{c^{(i)},(c')^{(i)}~:~i \in \{1,\dots,n\}\}$ let $S$ be the set of indices and $R$ be the set of pairs of indices, we let $\{p_{xh}(c)~:~c \in S\}$, $\{p_{hh}(c,c')~:~(c,c')\in R\}$, $\{p_{hy}(c)~:~c \in S\}$ be independent. Then we let
$
  \mathrm{Law}(p_{xh}(c)) = \rho_{xh}, \, \mathrm{Law}(p_{hh}(c,c')) = \rho_{hh},\, \mathrm{Law}(p_{hy}(c)) = \rho_{hy}
$
for all $c \in S$, $(c,c')\in R$. This space exists by Kolmogorov extension theorem. The desired neuronal embedding is obtained by  taking $\Omega_{h} = \Lambda \times [0,1]$, equipped with the measure $ P_{hh} = P_0\times \mathrm{Unif}([0,1])$ and we define the weight functions as
\begin{equs}
  \wxh((\lambda_1, c)) &= p_{xh}(c)(\lambda_1)\\
   \whh((\lambda_1, c),(\lambda_2,c')) &= p_{hh}(c,c')(\lambda_1,\lambda_2)\\
      \why((\lambda_2, c)) &=  p_{hy}(c)(\lambda_2).
\end{equs}
\end{example}
In order to state our results we assume that the dependence of $\theta(i)$ and $\theta(j)$ for $i\neq j$ is sufficiently weak, as stated in \aref{a:indep} below. While this condition is trivially satisfied by \abbr{iid} initialization introduced above, it makes our analysis applicable to more general -- not fully necessarily \abbr{iid} -- initialization procedures.


\subsection{Training dynamics}
A popular algorithm to minimize the MSE  \eref{e:mse1} is given by \emph{gradient descent}: starting from an initial condition
$\ww(0)$, 
we update the parameters $\ww$ in the direction of steepest descent of the loss function:
\begin{equ}\label{e:pgu1}
\mathbf W(j+1) := \mathbf W(j) - \beta D_{\mathbf W} \mathcal L(\mathbf W)\,,
\end{equ}
where $D_\ww$ represents the Fr\'echet derivative with respect to $\ww$, $j \in \mathbb N_0$ indexes the timesteps of the algorithm and $\beta$ denotes the stepsize of the discrete-time update.

In this work, we consider the  regime of asymptotically small constant
step-sizes, $\beta \to 0$. In this continuum limit, the stochastic
component of the dynamics is averaged before the parameters of the
model can change significantly.  This allows us to consider the parametric update as a deterministic dynamical system emerging from the
averaging of the underlying stochastic algorithm corresponding to the limit of infinite sample sizes.
This is known as the ODE
method \cite{Borkar:09} for analyzing stochastic approximations. We focus on the analysis
of this deterministic system to highlight the core dynamical properties of our training algorithm.

We denote
$$\mathbf W(t) := \{\wwxh(t; \cdot), \wwhh(t; \cdot,  \cdot), \wwhy(t; \cdot)\}$$
as the continuous-time, averaged trajectory of the finite-width weights with initial conditions
$\wwxh(0; \cdot)=\wwxh^0(\cdot)$, $\wwhh(0; \cdot,  \cdot)=\wwhh^0(\cdot,\cdot)$, $\wwhy(0; \cdot)=\wwhy^0(\cdot)$. The gradient descent dynamics for these quantities can be written as the following
\abbr{ode}s
\begin{equ}\label{e:gradientdescent}
    \partial_t \mathbf W(t) = - D_{\WW} \mathcal L(\mathbf W(t))\,.
\end{equ}
While the dynamics of both $\wwxh$ and $\wwhy$ will be described by the above equation, we truncate the evolution of $\wwhh$ in an interval of width $R>0$ as follows:
\begin{equs}\label{e:truncation1}
  \partial_t \mathbf W_{hh}(t) =& - \chi_R(\wwhh(t)) \odot D_{\WW_{hh}} \mathcal{L}(\mathbf W(t))
\end{equs}
where $\odot$ denotes the Hadamard product and  $\chi_R~:~\mathbb R\to \mathbb R$ is a smooth indicator function acting component-wise on its argument and such that $\chi_R(w) = w$ if $\|w\| \leq R/2$ and $\chi_R(w) \equiv 0$ if $\|w\| \geq R$. \aa{We comment on the reasons for this truncation in \rref{r:truncation}.}

Analogously, we denote
$$W(t) := \{\wxh(t; \cdot), \whh(t; \cdot,  \cdot), \why(t; \cdot)\}\,,$$
as the continuous-time trajectory of the mean-field weights with initial condition $\wxh(0; \cdot)=\wxh^0(\cdot)$, $\whh(0; \cdot,  \cdot)=\whh^0  (\cdot,\cdot)$, $\why(0; \cdot)=\why^0(\cdot)$, obeying the set of
\abbr{ode}s
\begin{equs}
 \partial_t \why(t;\theta) =& - \frac{\delta}{\delta \why} \mathcal L(W(t))\\
  \partial_t \whh(t;\theta, \theta') =& - \chi_R(\whh(t;\theta, \theta')) \frac{\delta}{\delta \whh} \mathcal L(W(t))\qquad\label{e:mfpde}\\
    \partial_t \wxh(t;\theta) = & - \frac{\delta}{\delta \wxh} \mathcal L(W(t))
\end{equs}
where $\frac\delta{\delta W}$ denotes the variational derivative (Fréchet derivative) with respect to $W$. While the explicit expressions for these dynamics are derived in Appendix~\ref{s:RHS}, we give here the update for the last layer of mean-field weights:
\begin{equs}\label{e:deltaw}
\partial_t  \why(t;\theta)  =& - \int  (\hat F(\mathbf x;W(t)) - F^*(\mathbf x))
\sigma_h\big(H_{hh}(\theta;\mathbf x,0) +  H_{xh}(\theta;\mathbf x_0)\big) \nu(\d \mathbf x)\,.
\end{equs}
In the next section we will leverage the fact that this quantity must be 0 at stationarity to establish the desired optimality result.
\section{Optimality results}\label{s:results}

To state the main results of this paper, \aa{denoting by $L_R^\infty(\phh)$ whe set of functions on $\Omega_{h}$ that are essentially bounded by $R>0$,} we formulate the following assumption:
\begin{assumption} \label{a:1} Consider a neuronal embedding $(\Ohh,\phh, W)$ and consider a mean-field limit associated with the neuronal ensemble $(\Ohh, \phh)$ with initialization $W(0) = W^0$. We assume that there exists $K>R$ such that
  \begin{enumerate}
    \item[a)] {\rm Regularity of $\sigma$:}
    $\sigma_h$ is bounded, differentiable, $\sigma_h(0) = 0$, $\sigma_h'(0)\neq 0$ and $D\sigma_h$ is $K$-bounded and $K$-Lipschitz.
    \item[b)] {\rm Universal approximation}: The span of $\{\sigma_h(\wxh\cdot \xx_0)\,:\,\wxh \in \mathbb R^\din \}$ is dense in $L^2(\nu_0)$.

    \item[c)] {\rm Diversity at initialization}: The support of the weight functions $\whh^0, \wxh^0$ at initialization satisfies
$$\supp(\wxh^0(\theta),\whh^0(\cdot,\theta),\whh^0(\theta,\cdot)) = \Pxh \times L_R^\infty(\phh) \times L_R^\infty(\phh)\,.$$
Throughout the paper we denote by $\whh(\cdot,\theta)$ the random (in $\theta$) mapping $\theta' \mapsto \whh(\theta',\theta)$.
\item[d)] {\rm Regularity at initialization}: 
The weight functions $\why^0, \whh^0, \wxh^0$ at initialization satisfy \aa{$\sup_{\theta, \theta'} |\whh^0(\theta, \theta')| \leq R$ and} given $E_1(m) = \mathbb E(|\wxh^0(\theta)^m|)^{1/m}$ and
$ E_2(m) =
\mathbb E(|\why^0(\theta)^m|)^{1/m}\quad $ then
\begin{equ}
    \sup_{m \geq 1} \frac 1 {\sqrt m}\pq{ E_1(m)
    \vee E_2(m)}< K\,.
\end{equ}
  \end{enumerate}
  \end{assumption}

 Most of the assumptions made above are standard in the literature on mean-field limits of neural networks, and were first formulated in similar terms in \cite{ChizatBach18} and \cite{pham20global}. \aref{a:1}a) gives technical conditions on the regularity of the nonlinearities, ensuring that the training dynamics are well-behaved. \aa{The condition on the nonvanishing derivative at the preimage of $0$, which without loss of generality is assumed to be at $0$ itself, is required to preserve expressivity of the network while allowing for uniform in time boundedness of the hidden weights.} \aref{a:1}b) demands sufficient expressivity of the activation function, required to approximate any function of a finite list of inputs $\{\xx_{-L},\dots,\mathbf x_0\}$. This condition replaces the convexity assumption from \cite{ChizatBach18}, and is satisfied by any nonlinearity for which the universal approximation theorem holds \cite{Cybenko:89,Barron:93}, \eg $\tanh$. \aref{a:1}c) guarantees that the initial condition is such that the expressivity from b) can actually be exploited. This property, which as we shall show is preserved by the network throughout training, ensures that the argument of the nonlinearity at each layer is sufficiently varied, and was first introduced in \cite{pham20global}. Combining this with \aref{a:1}b) ensures, by induction, that there is no information bottleneck throughout the depth of the unrolled network and that the model is highly expressive throughout training. Finally, \aref{a:1}d) is a technical assumption on the data and on the weights guaranteeing the well-posedness of the training dynamics.

 \begin{remark}\label{r:iid}
  We note that \aref{a:1}c) is significantly stronger than the analogous ``sufficient support'' assumption from \cite{ChizatBach18}. In particular, this assumption is \emph{not} satisfied if the weights of each layer are sampled \abbr{iid} from \emph{any} initialization law $\mu$. As we comment in the proof of our results, relaxing this assumption to include \abbr{iid} initialization would significantly reduce the expressivity of the untrained infinite-width network with respect to predictors $\xx_k$ at timesteps $k<0$. More specifically, an \abbr{iid} initialization of the weights combined with the infinite width limit we are considering results in a highly degenerate hidden state of the network. Because of the intrinsic depth of RNN structures, this generates in turn a bottleneck effect preventing information from values of the predictors in the distant past to propagate through the network.
 \end{remark}

 \begin{remark}\label{r:truncation} \aa{The truncation of the dynamics of the hidden layer weights \eref{e:truncation1} \eref{e:mfpde} was introduced in order to guarantee existence and uniqueness of the solution to both the finite-width and the mean-field equations. Indeed, in the absence of this cutoff, weight-sharing in this class of RNNs would result in a non-Lipschitz RHS for the dynamical equation \eref{e:gradientdescent}, as shown explicitly in  Appendix~\ref{s:RHS}. Given this lack of regularity, existence of the solution cannot be guaranteed by standard analytical tools. However, in practice the weights are stored using a floating-point representation which is intrinsically bounded, and we argue that in this sense the truncation of their trajectories is a relatively natural assumption.}
\end{remark}

We now proceed to present the main results of the paper, which we divide into two parts:

\subsection{Convergence}

The main result in this section is the convergence of the finite-width network trajectories to the mean-field limit, analogously to
Thm.~18 in \cite{pham20global}. More specifically, for a given neuronal ensemble $(\Omega, \pp)$ and sample $\mathbf W$ from $\pp$ we define the following distance or error metric
$\mathcal D_\tau(W, \mathbf W)$ for any $\tau >0$ as
\begin{equs}
  &\mathcal D_\tau(W, \mathbf W):= \sup_{t \in (0,\tau)}\left(\frac 1{n^2}\| \whh(t;\theta(i), \theta(j)) - \wwhh(t;i,j)\|_2 \right.\vee \frac 1 {n} \|\wxh(t;\theta(j))-\wwxh(t; j)\|_2\\
  & \qquad \qquad \qquad\qquad \qquad  \vee\left. \frac 1 {n} \|\why(t;\theta(j))-\wwhy(t; j)\|_2\right)
\end{equs}
where $\|\,\cdot\,\|_2$ denotes, depending on its argument, the Frobenius norm or the classical $\ell_2$ norm. 

\begin{theorem}\label{t:convergence}
For any $R>0$, let Assumptions~\ref{a:01a}, \ref{a:02}, \ref{a:1} and \ref{a:indep} hold. There exist constants $c,c'>0$ such that, for any $\delta>0$, any $L \in \mathbb N$ and $\tau >0$, there exists $n^*\in \mathbb N$ such that for any $n>n^*$ with probability at least $1-\delta-\bar K n\exp(-\bar K n^{c'})$ we have
\begin{equ}
  \mathcal D_\tau(W, \mathbf W) \leq \bar K n^{-c} \sqrt{\log\pc{n^2/\delta+e}}
\end{equ}
where $\bar K$ is a constant that depends on  $L$ and $R$.
\end{theorem}

The proof of the above result mimics the one in \cite{pham20global} and is provided in the appendix for completeness. The main argument of the proof is similar to classical propagation of chaos results \cite{sznitman91}. The first step of the argument establishes sufficient regularity of the gradient dynamics and guarantees existence and uniqueness of the solution to \eref{e:mfpde}. Then, one bounds the difference in differential updates for the particle system and the mean-field dynamics as a function of the distance $\mathcal D_t(W, \ww)$.
The proof is concluded by an application of Gr\"onwall's inequality.

\subsection{Optimality}

The main optimality result is presented in the following theorem.
\begin{theorem}\label{t:optimality}
  For any $R>0$ let Assumption~\ref{a:01a}, \ref{a:02} and \ref{a:1} hold and assume that the trajectory $W(t)$ solving \eqref{e:mfpde} converges to $\bar W$ \aa{in the following sense: for all $i \in \{1, \dots, L\}$ the following quantities vanish in the limit $t \to \infty$,
  \begin{equs}
    & \bullet \,\,\mathrm{ess\text{-}sup}_{\theta \in \mathrm{supp}(\phh)} |\partial_t W_{hy}(t;\theta)|  \\
    &\bullet \int |\bar W_{hy}(\theta) - W_{hy}(t;\theta)|^2 \phh(d\theta) \\
    &\bullet \int \bar W_{hy}(\theta^{(0)})^2 \pc{\prod_{j=1}^{i-1} \bar W_{hh}(\theta^{(j-1)},\theta^{(j)})}^2 \pc{\bar W_{hh}(\theta^{(i-1)},\theta^{(i)})- W_{hh}(t;\theta^{(i-1)},\theta^{(i)})}^2P_{hh}^{\otimes i+1}(\d \theta^{(0)},..., \d \theta^{(i)})   \\
    &\bullet \int \bar W_{hy}(\theta^{(0)})^2 \pc{\prod_{j=1}^{i-1} \bar W_{hh}(\theta^{(j-1)},\theta^{(j)})}^2\pc{\bar W_{xh}(\theta^{(L-1)})- W_{xh}(t;\theta^{(L-1)})}^2\phh^{\otimes i+1}(\d \theta^{(0)},..., \d \theta^{(i)})
    \end{equs}
    }
    Then $\lim_{t \to \infty} \mathcal L(W(t)) = 0$\,.
\end{theorem}

This result asserts that if the gradient descent dynamics \eref{e:mfpde} converges to a stationary point $\bar W$, 
that point must be a global minimizer, \ie it must approximate the underlying function to arbitrary accuracy.
We prove the above result in three steps. First,
we show in \pref{p:spanning} that if the weights at initialization are sufficiently varied (\aref{a:1}c)) then the network enjoys a high level of expressivity, inherited from the properties of $\sigma_h$ \aref{a:1}a) and b). Such expressivity in turn implies that the mean-field vector fields evaluated at a suboptimal fixed point of the dynamics \eref{e:mfpde} cannot vanish everywhere
in neuronal embedding space.
 In other words, a network whose weights have sufficient support cannot correspond to a suboptimal stationary point of the gradient dynamics.

We then show in \lref{l:support} that such sufficient notion of support (\aref{a:1}c)) is preserved by the gradient descent dynamics \eref{e:mfpde} throughout training.
For any finite time, this is true by topological arguments: the full support property cannot be altered by a continuous vector field such as \eref{e:mfpde}.

Finally, we show that the gradient descent dynamics cannot converge to a spurious fixed point by combining the two partial results above.
In particular, we show that by the preserved expressivity of the network throughout the dynamics proven above, the fact that the time derivative of $W_{hy}(t)$ \eref{e:deltaw} must vanish almost-everywhere as $t \to \infty$ implies that the difference between the approximator and the target function $F^*$ must also vanish almost everywhere in the limit. In other words, combining the assumption on convergence of $W_{hy}(t)$ with the nondegeneracy of the $W_{hy}$-Jacobian of the network (following from expressivity) imples that the limiting point must be optimal.

There are multiple technical challenges that need to be addressed in this proof with respect to the proof techniques used in previous results. The most important one stems from the fact that the input structure of the (unrolled) network is different from a standard feedforward network or ResNet. The additive combination of the input with the hidden state of the previous ``layer'', together with weight sharing, results in possible degeneracies of the dynamics that need to be taken into account in the proof.
By studying the risk minimization problem in equation \eref{e:mse1} and considering exclusively the dynamics of $\why$ \eref{e:deltaw}, we bypass the problem of weight sharing in the unrolled network by leveraging the expressivity \aref{a:1}c).

\aa{Finally, we note that the boundedness of $\whh$ prevents us from using any of the classical expressivity results leveraging the vector space structure of the space of admissible weights. For this reason, we adapt our proof to bypass the boundedness of the hidden weights resulting from the truncation in \eref{e:mfpde}. To do so, we leverage the fact that, by \aref{a:1}a), the image under $\sigma_h$ of a function whose supremum is close to $0$ is close to the identity. Combining this with the possibility of choosing arbitrarily small hidden weights results in the network being able to propagate information throughout its layers. Finally, the unboundedness of $\why$ allows to recover this information and therefore to realize the expressive potential of the network.}

\section{Numerical experiments}\label{s:numerics}

In this section we numerically validate the
theoretical
optimality and convergence results in our paper with some simulations.

\begin{figure}[t!]
  \centering
    \begin{subfigure}{.49\textwidth}
    \centering
    \includegraphics[width = 0.99\linewidth]{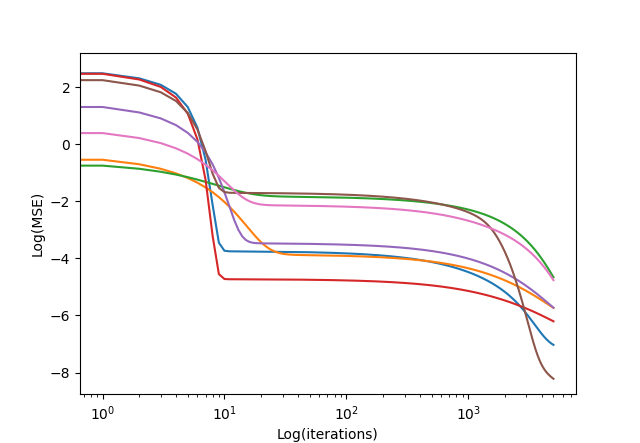}
    \caption{}
    \label{fig:sub1}
  \end{subfigure}
  \begin{subfigure}{.49\textwidth}
    \centering
    \includegraphics[width = 0.99\textwidth]{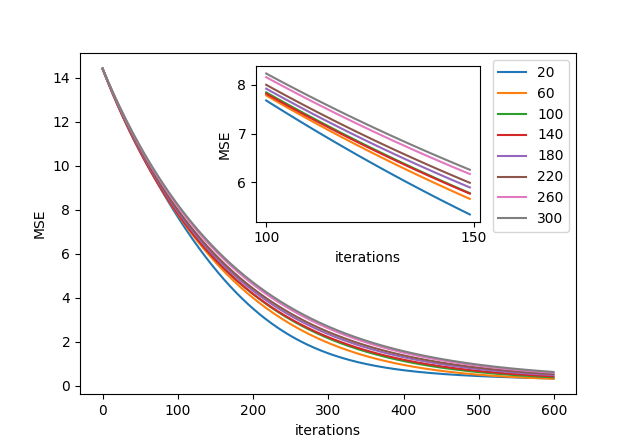}
    \caption{}
    \label{fig:sub3}
  \end{subfigure}
    \caption{{Results of numerical experiments. In \fref{fig:sub1} we plot evolution of the MSE $\tilde {\mathcal L}(\hat{\mathbf W})$ on a log-log scale as a function of training steps for predictors generated by a deterministic  
    dynamical system.
    Here, }different curves correspond to different initializations of the student network. For each experiment, the data, the weights of the teacher and the initialization weights of the student are generated anew independently from previous experiments. In \fref{fig:sub3}, we plot the evolution of the MSE for a network with growing size (the legend indicates the size of the student network). The inset zooms on the same plots for timesteps $k \in [100,150]$.}
    \label{f:numerics}
  \end{figure}

\paragraph{Network architecture  and model}

The network model we consider is a wide a RNN in the mean-field regime in a teacher-student scenario. For the optimality results we set the width to be
$n_s = 1000$ and for the convergence results we  train the student RNN with increasing size ($n_s \in \{20,60,100,140,\dots, 300\}$).

We train a so-called \emph{student} many-to-one RNN $\hat F$ with $d = 1$ and hidden layer width $n_s$ to learn the output of a \emph{teacher} many-to-one RNN $F^*$, with the same input and output size and hidden width $n_t = 15$. Both neural networks have hidden activation $\sigma_h(\cdot) = \tanh(\cdot)$
and their weights are initialized \abbr{iid} as follows:
\begin{equs}
  \text{teacher:} \qquad\qquad\qquad\qquad&\text{student:}\\\begin{cases}{\mathbf W}_{xh} \sim \mathcal N(1,1)\\ {\mathbf W}_{hh} \sim \mathcal N(0,n_t^{-2})\\ {\mathbf W}_{hy} \sim \mathcal N(0,n_t^{-2})\end{cases}\quad&
   \begin{cases}{\mathbf W}_{xh} \sim \mathcal N(0,5)\\ {\mathbf W}_{hh} \sim \mathcal N(0,10 \cdot n_s^{-2})\\ {\mathbf W}_{hy} \sim \mathcal N(0,10 \cdot n_s^{-2}).\end{cases}
\end{equs}

\paragraph{Simulated data}

        The predictors for our 
        simulation are generated as samples of length $L = 10$ from the stationary trajectories of the shift map $T(x) = x+1$ acting on the sphere $\mathbb X = S^1 = [0, 2\pi)$. To do so, we sample the initial point $\mathbf x_{-L}^{(j)}$ \abbr{iid} from the invariant measure $\nu_0(\d x) = \frac1{2\pi}\d x$ of $T$ supported on $\mathbb X$ and generate the corresponding input sequence as
        $\mathbf x_{-k+1}^{(j)} := T(\mathbf x_{-k}^{(j)}) $
        for $k \in (1,,\dots, L)$.

 \paragraph{Training specifications} The training of the student RNN is performed using the $\texttt{nn}$ package in pytorch \cite{pytorch}. We train the parameters $\hat{\mathbf W}$ to minimize the empirical  Mean Squared Error $\tilde {\mathcal L}(\hat{\mathbf W}):=\frac 1 {m} \sum_{j = 1}^m (F^*(\mathbf x^{(j)}) - \hat F(\mathbf x^{(j)},\mathbf W ))^2$ where $m = 2^{13}\approx 10^4$
 denotes the size of the database. Combining this with our sampling of $\mathbf x^{(j)}$ results in \abbr{iid} samples from the invariant measure $\nu(\mathbf x)$ of $T$, and therefore in the finite-sample equivalent of the population risk \eref{e:mse1}. The optimization is performed using stochastic gradient descent (\texttt{pytorch.optim.SGD}), which is called with a stepsize $\gamma = 3 10^{-3}$ and batch size of $m$, the full database size, thereby resulting in full-fledged gradient descent. The results of the simulation are shown in \fref{f:numerics}. Code is available at \cite{github}.

   \paragraph{Initialization for the convergence results}

    For the convergence results we need to
    initialize the family of student networks in a consistent way. Our procedure for enforcing consistency draws the weights without replacement from a reference student network of width $300$.

\section{Conclusions}\label{s:conclusion}

This work shows that, despite the increased complexity,  RNNs  share common optimality properties with simpler single-layer neural networks \cite{ChizatBach18}. Specifically we show that, under some conditions on the expressivity of the network at initialization, the fixed points of wide Elman-type RNNs' with gradient descent training dynamics in the mean-field regime are globally optimal, \ie that the neural network will perfectly learn the given function of the dynamical system's trajectories. In this sense, while extending previous results on the optimality properties of shallow and deep neural networks to novel architectures, this work contributes to the understanding of deep learning applied to dynamical systems data. The proof is carried out by unrolling the RNN structure and showing that the fixed points of the training dynamics, which preserve a certain notion of support in parameter space, can only be globally optimal.

Possible future developments include relaxing the assumption about the support of the weights at initialization \aref{a:1}c) to a condition that is simpler to realize in practice. Drawing an analogy with autoregressive processes, a promising insight towards solving this problem consists in injecting, at each iteration of the RNN, new directions in the function space spanned by the model by means of, \eg the network biases, so as to reduce the hidden layer's null space. Another possible avenue of future research consists of relaxing the adiabaticity assumption, \ie considering the stochastic approximation problem resulting from the finite number of samples and the finite gradient stepsize. We note that, because of this assumption, our analysis is immune to the exploding gradients problem \cite{pascanu13}. To prevent this problem to affect a finite timestep analysis, another important extension of the present work is to establish similar results for different RNN architectures, such as the LSTM, which given its extensive use in practice is of great interest.

From the theoretical standpoint, the most important
open question concerns establishing quantitative convergence of mean-field dynamics of neural networks: even in the single-layer, supervised setting, despite recent results in specific settings \cite{chizat19}, these guarantees still elude the community's research efforts.

\subsubsection*{Acknowledgments.}
 We thank the anonymous referee for pointing out a gap in our previous proof of \tref{t:eandu}. All authors acknowledge partial support of the TRIPODS NSF grant CCF-1934964. AA acknlwledges partial support of the University of Pisa, through project PRA $2022\_85$. JL acknowledges the partial support of the NSF grant DMS-2012286. SM acknowledges partial support of HFSP RGP005, NSF
DMS 17-13012, NSF BCS 1552848, NSF DBI 1661386, NSF IIS 15-46331, NSF DMS 16-
13261, as well as high-performance computing partially supported by grant 2016-IDG-1013 from the North Carolina Biotechnology Center. AA and SM thank Katerina Papagiannouli and Andrea Aveni for insightful discussions and acknowledge the hospitality of the Max Planck Institute for Mathematics in the Sciences and of the ScaDS institute of the University of Leipzig and Technical University of Dresden during the final part of this project.

\bibliographystyle{plain}

\bibliography{bib}

\begin{thebibliography}{10}

\bibitem{agazzi20}
Andrea Agazzi and Jianfeng Lu.
\newblock Global optimality of softmax policy gradient with single hidden layer
  neural networks in the mean-field regime.
\newblock In {\em International Conference on Learning Representations}, 2020.

\bibitem{agazzi19}
Andrea Agazzi and Jianfeng Lu.
\newblock Temporal-difference learning with nonlinear function approximation:
  lazy training and mean field regimes.
\newblock In Joan Bruna, Jan Hesthaven, and Lenka Zdeborova, editors, {\em
  Proceedings of the 2nd Mathematical and Scientific Machine Learning
  Conference}, volume 145 of {\em Proceedings of Machine Learning Research},
  pages 37--74. PMLR, 16--19 Aug 2022.

\bibitem{rntk20}
Sina Alemohammad, Zichao Wang, Randall Balestriero, and Richard Baraniuk.
\newblock The recurrent neural tangent kernel.
\newblock In {\em International Conference on Learning Representations}, 2021.

\bibitem{github}
anonymous.
\newblock {Code for numerical simulations}.
\newblock url{https://github.com/anonymous229321329857123/rnn}.

\bibitem{Barron:93}
A.~R. {Barron}.
\newblock Universal approximation bounds for superpositions of a sigmoidal
  function.
\newblock {\em IEEE Transactions on Information Theory}, 39(3):930--945, 1993.

\bibitem{Borkar:09}
Vivek~S Borkar.
\newblock {\em Stochastic approximation: a dynamical systems viewpoint},
  volume~48.
\newblock Springer, 2009.

\bibitem{chen18}
Minmin Chen, Jeffrey Pennington, and Samuel Schoenholz.
\newblock {Dynamical isometry and a mean field theory of RNNs: Gating enables
  signal propagation in recurrent neural networks}.
\newblock In {\em International Conference on Machine Learning}, pages
  873--882. PMLR, 2018.

\bibitem{chizat19}
Lenaic Chizat.
\newblock Sparse optimization on measures with over-parameterized gradient
  descent.
\newblock {\em Mathematical Programming}, 194(1-2):487--532, 2022.

\bibitem{ChizatBach18}
L{\'e}na\"{\i}c Chizat and Francis Bach.
\newblock {On the Global Convergence of Gradient Descent for Over-parameterized
  Models Using Optimal Transport}.
\newblock In {\em Proceedings of the 32Nd International Conference on Neural
  Information Processing Systems}, NIPS'18, pages 3040--3050, USA, 2018. Curran
  Associates Inc.

\bibitem{ChizatBach182}
L{\'e}na\"{\i}c Chizat and Francis Bach.
\newblock {On the Global Convergence of Gradient Descent for Over-parameterized
  Models Using Optimal Transport}.
\newblock In {\em NIPS 31}, 2018.

\bibitem{ChizatBach20}
Lenaic Chizat and Francis Bach.
\newblock Implicit bias of gradient descent for wide two-layer neural networks
  trained with the logistic loss.
\newblock In {\em Conference on Learning Theory}, pages 1305--1338. PMLR, 2020.

\bibitem{GRUs1}
Kyunghyun Cho, Bart Van~Merri{\"e}nboer, Dzmitry Bahdanau, and Yoshua Bengio.
\newblock {On the properties of neural machine translation: Encoder-decoder
  approaches}.
\newblock {\em arXiv preprint arXiv:1409.1259}, January 2014.

\bibitem{Cybenko:89}
G.~Cybenko.
\newblock Approximation by superpositions of a sigmoidal function.
\newblock {\em Mathematics of Control, Signals and Systems}, 2(4):303--314, Dec
  1989.

\bibitem{mlfinance}
Matthew~F Dixon, Igor Halperin, and Paul Bilokon.
\newblock {\em {Machine learning in Finance}}, volume 1170.
\newblock Springer, 2020.

\bibitem{Elman90}
J.L. Elman.
\newblock {Finding Structure in Time}.
\newblock {\em Cognitive Science}, 14:179--211, 1990.

\bibitem{schmidhuber2}
Felix~A Gers, J{\"u}rgen Schmidhuber, and Fred Cummins.
\newblock {Learning to forget: Continual prediction with {LSTM}}.
\newblock {\em Neural Computation}, 12(10):2451--2471, 2000.

\bibitem{Montanari:19c}
Behrooz Ghorbani, Song Mei, Theodor Misiakiewicz, and Andrea Montanari.
\newblock {Limitations of Lazy Training of Two-layers Neural Networks}.
\newblock In {\em Advances in Neural Information Processing Systems}, pages
  9108--9118, 2019.

\bibitem{Montanari:19b}
Behrooz Ghorbani, Song Mei, Theodor Misiakiewicz, and Andrea Montanari.
\newblock {Linearized two-layers neural networks in high dimension}.
\newblock {\em The Annals of Statistics}, 49(2):1029 -- 1054, 2021.

\bibitem{schmidhuber1}
Sepp Hochreiter and J{\"u}rgen Schmidhuber.
\newblock Long short-term memory.
\newblock {\em Neural computation}, 9(8):1735--1780, 1997.

\bibitem{hongler18}
Arthur Jacot, Franck Gabriel, and Cl{\'e}ment Hongler.
\newblock {Neural tangent kernel: Convergence and generalization in neural
  networks}.
\newblock In {\em Advances in Neural Information Processing Systems}, pages
  8571--8580, 2018.

\bibitem{jordan1}
Michael~I Jordan.
\newblock Serial order: A parallel distributed processing approach.
\newblock In {\em Advances in psychology}, volume 121, pages 471--495.
  Elsevier, 1997.

\bibitem{dstheory}
Anatole Katok and Boris Hasselblatt.
\newblock {\em Introduction to the Modern Theory of Dynamical Systems}.
\newblock Encyclopedia of Mathematics and its Applications. Cambridge
  University Press, 1995.

\bibitem{lecun89}
Yann LeCun, Bernhard Boser, John~S Denker, Donnie Henderson, Richard~E Howard,
  Wayne Hubbard, and Lawrence~D Jackel.
\newblock Backpropagation applied to handwritten zip code recognition.
\newblock {\em Neural Computation}, 1(4):541--551, 1989.

\bibitem{lu20}
Yiping Lu, Chao Ma, Yulong Lu, Jianfeng Lu, and Lexing Ying.
\newblock {A mean field analysis of deep ResNet and beyond: Towards provably
  optimization via overparameterization from depth}.
\newblock In {\em International Conference on Machine Learning}, pages
  6426--6436. PMLR, 2020.

\bibitem{MeiMonNgu18}
Song Mei, Andrea Montanari, and Phan-Minh Nguyen.
\newblock A mean field view of the landscape of two-layer neural networks.
\newblock {\em Proceedings of the National Academy of Sciences},
  115(33):E7665--E7671, 2018.

\bibitem{pham20global}
Phan-Minh {Nguyen} and Huy~Tuan {Pham}.
\newblock {A Rigorous Framework for the Mean Field Limit of Multilayer Neural
  Networks}.
\newblock {\em arXiv e-prints}, page arXiv:2001.11443, January 2020.

\bibitem{pascanu13}
Razvan Pascanu, Tomas Mikolov, and Yoshua Bengio.
\newblock On the difficulty of training recurrent neural networks.
\newblock In {\em International Conference on Machine Learning}, pages
  1310--1318. PMLR, 2013.

\bibitem{pytorch}
Adam Paszke, Sam Gross, Francisco Massa, Adam Lerer, James Bradbury, Gregory
  Chanan, Trevor Killeen, Zeming Lin, Natalia Gimelshein, Luca Antiga, Alban
  Desmaison, Andreas Kopf, Edward Yang, Zachary DeVito, Martin Raison, Alykhan
  Tejani, Sasank Chilamkurthy, Benoit Steiner, Lu~Fang, Junjie Bai, and Soumith
  Chintala.
\newblock {PyTorch: An Imperative Style, High-Performance Deep Learning
  Library}.
\newblock In H.~Wallach, H.~Larochelle, A.~Beygelzimer, F.~d\textquotesingle
  Alch\'{e}-Buc, E.~Fox, and R.~Garnett, editors, {\em Advances in Neural
  Information Processing Systems 32}, pages 8024--8035. Curran Associates,
  Inc., 2019.

\bibitem{pham21global}
Huy~Tuan Pham and Phan-Minh Nguyen.
\newblock {Global Convergence of Three-layer Neural Networks in the Mean Field
  Regime}.
\newblock In {\em International Conference on Learning Representations}, 2020.

\bibitem{mlinm}
Alvin Rajkomar, Jeffrey Dean, and Isaac Kohane.
\newblock Machine learning in medicine.
\newblock {\em New England Journal of Medicine}, 380(14):1347--1358, 2019.

\bibitem{RotVE18}
Grant Rotskoff and Eric Vanden-Eijnden.
\newblock Parameters as interacting particles: long time convergence and
  asymptotic error scaling of neural networks.
\newblock In S.~Bengio, H.~Wallach, H.~Larochelle, K.~Grauman, N.~Cesa-Bianchi,
  and R.~Garnett, editors, {\em Advances in Neural Information Processing
  Systems 31}, pages 7146--7155. Curran Associates, Inc., 2018.

\bibitem{rumelhart85}
David~E Rumelhart, Geoffrey~E Hinton, and Ronald~J Williams.
\newblock Learning internal representations by error propagation.
\newblock Technical report, California Univ San Diego La Jolla Inst for
  Cognitive Science, 1985.

\bibitem{Silver:16}
David Silver, Aja Huang, Christopher~J. Maddison, Arthur Guez, Laurent Sifre,
  George van~den Driessche, Julian Schrittwieser, Ioannis Antonoglou, Veda
  Panneershelvam, Marc Lanctot, Sander Dieleman, Dominik Grewe, John Nham, Nal
  Kalchbrenner, Ilya Sutskever, Timothy Lillicrap, Madeleine Leach, Koray
  Kavukcuoglu, Thore Graepel, and Demis Hassabis.
\newblock Mastering the game of {G}o with deep neural networks and tree search.
\newblock {\em Nature}, 529:484--503, 2016.

\bibitem{spiliopoulos18}
Justin Sirignano and Konstantinos Spiliopoulos.
\newblock Mean field analysis of deep neural networks.
\newblock {\em Mathematics of Operations Research}, 2021.

\bibitem{spiliopoulos22}
Justin Sirignano and Konstantinos Spiliopoulos.
\newblock Asymptotics of reinforcement learning with neural networks.
\newblock {\em Stochastic Systems}, 12(1):2--29, 2022.

\bibitem{sznitman91}
Alain-Sol Sznitman.
\newblock Topics in propagation of chaos.
\newblock In {\em Ecole d'{\'e}t{\'e} de probabilit{\'e}s de Saint-Flour
  XIX—1989}, pages 165--251. Springer, 1991.

\bibitem{vinyals19}
Oriol Vinyals, Igor Babuschkin, Wojciech~M Czarnecki, Micha{\"e}l Mathieu,
  Andrew Dudzik, Junyoung Chung, David~H Choi, Richard Powell, Timo Ewalds,
  Petko Georgiev, et~al.
\newblock {Grandmaster level in StarCraft II using multi-agent reinforcement
  learning}.
\newblock {\em Nature}, 575(7782):350--354, 2019.

\bibitem{wojtowytsch20}
Stephan Wojtowytsch.
\newblock {On the Convergence of Gradient Descent Training for Two-layer
  ReLU-networks in the Mean Field Regime}.
\newblock {\em arXiv preprint arXiv:2005.13530}, 2020.

\bibitem{yang20}
Greg Yang.
\newblock {Tensor programs ii: Neural tangent kernel for any architecture}.
\newblock {\em arXiv preprint arXiv:2006.14548}, 2020.

\end{thebibliography}

\appendix
\onecolumn

\numberwithin{equation}{section}
\renewcommand\thefigure{\thesection.\arabic{figure}}

\section{Computation of mean-field ODEs}\label{s:RHS}

In this section we explicitly compute the RHS of the mean-field ODEs \eref{e:mfpde}. Recall the definitions of the population risk $\mathcal L(W)$ from \eref{e:mse1} and of the mean-field approximator
\begin{equ}\label{e:amfrnn}
\begin{aligned}
\hat F(\mathbf x;W) & = 
H_{hy}(\mathbf x) \\
H_{hy}(\mathbf x) & = \int \why(\theta) \sigma_h(H_{hh}(\theta; \mathbf x,0) + H_{xh}(\theta; \mathbf x_0)) \phh(\d\theta)\\
H_{hh}(\theta; \mathbf x,k) & = \int \whh(\theta, \theta') \sigma_h(H_{hh}(\theta';\mathbf x,k+1) + H_{xh}(\theta';\mathbf x_{-(k+1)})) \phh(\d \theta')\\
H_{xh}(\theta; \mathbf x_{-k}) & = \wxh(\theta) \cdot \mathbf x_{-k}
\end{aligned}
\end{equ}
Further, for notational convenience, we define throughout the argument of the nonlinearity as
\begin{equ}\label{e:hh}
  \HH(\theta; \xx, k) := H_{hh}(\theta;\mathbf x,k) + H_{xh}(\theta;\mathbf x_{-k})
\end{equ}
and, when necessary, we will slightly abuse notation and  explicitly write the set of weights generating the hidden state $\HH$ in its argument as $\HH[W](\theta; \xx, k)$.
Furthermore, we define
\begin{equ}\label{e:Df}
  \Delta F(W, \xx) := \hat F(\mathbf x;W(t)) - F^*(\mathbf x)
\end{equ}
so that we can write
\begin{equs}
\frac \delta {  \delta \why}\mathcal L[W](\theta)  & = \int \Delta F(W, \xx)  
\sigma_h\big(  \HH[W](\theta; \xx, 0)\big) \nu(\d \mathbf x)
\end{equs}

We proceed to compute the derivative \abbr{wrt}  $\wxh$:
\begin{equs}
\frac \delta {  \delta \wxh}\mathcal L[W](\theta)  &= \int  \Delta F(W, \xx)  \pq{\frac \delta {\delta \wxh} \int \why(\theta') \sigma_h\big(H_{hh}(\theta';\mathbf x,0) + H_{xh}(\theta';\mathbf x_0)\big)\phh(\d \theta')}(\theta)  \nu(\d \mathbf x)\\
&= \int  \Delta F(W, \xx)
 \int \why(\theta') \Xi_0(\theta;\theta', \xx) \phh(\d \theta') \nu(\d \xx)
\end{equs}
where, denoting here and throughout by $\delta(\theta)$ the Dirac delta distribution,  we define recursively
\begin{equs}
  \Xi_i[W](\theta;\theta', \xx)&:= \frac \delta {\delta \wxh}\sigma_h\big(H_{hh}(\theta';\mathbf x,i) + H_{xh}(\theta';\mathbf x_{-i})\big)(\theta)\\
&  = \sigma_h'\big(  \HH(\theta'; \xx, i)\big)\pc{\pq{\frac \delta {\delta \wxh}H_{hh}(\theta';\mathbf x,i)}(\theta) + \xx_i\delta(\theta'-\theta)}\\
&  = \sigma_h'\big(  \HH(\theta'; \xx, i)\big)\pc{\int \whh(\theta', \theta'')\Xi_{i-1}(\theta; \theta'', \xx) \phh(\d \theta'') + \xx_i \delta(\theta'-\theta)}
\end{equs}
and throughout we slightly abuse notation by suppressing the dependency of $\Xi_i$ on $W$ when clear from the context.
Therefore,  we obtain
\begin{equs}\label{e:deltawhy}
\frac \delta {  \delta \wxh}\mathcal L[W](\theta)
&= \int  \Delta F(W, \xx)
 \left(
 \sum_{i=0}^L
\Gamma_i(W, \theta, \xx)\xx_{-i}\right)\nu(\d \mathbf x)
\end{equs}
where for $i \in \{0,1,\dots, L\}$ we define
\begin{equs}\label{e:gammaj}
  \Gamma_i(W, \theta,\xx) &= \int \why(\theta_0) \sigma_h'(\HH(\theta_0; \xx, 0))\int \whh(\theta_0,\theta_1)\sigma_h'(\HH(\theta_1; \xx, 1))\\
  &\qquad\qquad
  \dots\int \whh(\theta_i,\theta)\sigma_h'(\HH(\theta; \xx, i))\,\phh^{\otimes i+1}(\theta_0,\dots, \theta_i)
\end{equs}
Analogously, we proceed to compute the derivative \abbr{wrt} $\whh$:
\begin{equs}
\frac \delta {  \delta \whh}\mathcal L[W](\theta, \theta')  &= \int  \Delta F(W, \xx)  \pq{\frac \delta {\delta \whh} \int \why(\theta_0) \sigma_h\big(\HH[W](\theta_0;\mathbf x,0) \big)\phh(\d \theta_0)}(\theta,\theta')  \nu(\d \mathbf x)\\
&= \int  \Delta F(W, \xx)
 \int \why(\theta_0) \Xi_0'[W](\theta, \theta';\theta_0, \xx) \phh(\d \theta_0) \nu(\d \xx)
\end{equs}
where we define recursively
\begin{equs}
  \qquad\Xi_i'[W](\theta, \theta';\theta_i , \xx)&:= \frac \delta {\delta \whh}\sigma_h\big(H_{hh}[W](\theta_i ;\mathbf x,i) + H_{xh}[W](\theta_i;\mathbf x_{-i})\big)(\theta, \theta')\\
&  = \sigma_h'\big(\HH[W](\theta_i; \xx, i)\big)\pq{\frac \delta {\delta \whh}H_{hh}[W](\theta_i;\mathbf x,i)}(\theta, \theta') \\
&  = \sigma_h'\big(\HH[W](\theta_i; \xx, i)\big)\left(\Bigg.\int \whh(\theta_i, \theta_{i+1})\Xi_{i+1}'[W](\theta, \theta'; \theta_{i+1}, \xx) \phh(\d \theta_{i+1})\right. \\
&\quad \qquad\qquad \qquad\qquad \left. \Bigg. + \int \sigma_h\big(\HH[W](\theta_{i+1}; \xx, i+1)\big) \delta(\theta_i-\theta)\delta(\theta_{i+1}-\theta') \phh(\d \theta_{i+1})\right)
\end{equs}
and throughout we slightly abuse notation by suppressing the dependency of $\Xi_i'$ on $W$ when clear from the context.
Therefore, we obtain
\begin{equ}\label{e:deltawhh}
\frac \delta {  \delta \whh}\mathcal L[W](\theta,\theta')
= \int  \Delta F(W, \xx)
 \left(\sum_{i=0}^L\Gamma_i(W, \theta, \xx)\sigma_h(\HH(\theta'; \xx, i+1))\right)\nu(\d \mathbf x)
\end{equ}
for $\Gamma_i$ defined in \eref{e:gammai}.

\section{Existence and uniqueness of solutions to ODEs}\label{s:existenceuniqueness}

We now proceed to sketch the proof of existence and uniqueness of the solutions to the mean-field ODEs, stated below. To this aim, fixing throughout a value of the cutoff $R>0$ for \eref{e:mfpde}, we define the sub-Gaussian norm
\begin{equs}
  \llbracket W_{hh}  \rrbracket_{\psi,t} &:= \sqrt{50} \sup_{m \geq 1} \frac 1 {\sqrt{m}}\pc{\int \sup_{s<t} |W_{hh}(s,\theta,
  \theta')|^mP_{hh}^{\otimes 2}(d\theta, d\theta')}^{1/m}\\
  \llbracket W_{hy}  \rrbracket_{\psi,t} &:= \sqrt{50} \sup_{m \geq 1} \frac 1 {\sqrt{m}} \pc{\int\sup_{s<t} |W_{hy}(s,\theta)|^m \phh(d\theta)}^{1/m}\\
  \llbracket W_{xh}  \rrbracket_{\psi,t} &:= \sqrt{50} \sup_{m \geq 1} \frac 1 {\sqrt{m}} \pc{\int\sup_{s<t} |W_{xh}(s,\theta)|^m \phh(d\theta)}^{1/m}
\end{equs}
inducing the norm on the weights $W$
\begin{equ}
  \llbracket W \rrbracket_{\psi,t} := \max\pc{\llbracket W_{hh}\rrbracket_{\psi,t}, \llbracket W_{xh}\rrbracket_{\psi,t}, \llbracket W_{hy}  \rrbracket_{\psi,t}}\,.
\end{equ}
From these definitions we have that $\llbracket W_{hh}  \rrbracket_{\psi,t} \geq \|W_{hh}\|_t$, $\llbracket W_{hy}  \rrbracket_{\psi,t} \geq \|W_{hy}\|_t$,  $\llbracket W_{xh}  \rrbracket_{\psi,t} \geq \|W_{xh}\|_t$
where
\begin{equs}
   \|\whh\|_t& =\pc{\int \sup_{s\leq t} |\whh(s,\theta,\theta')|^{50}\phh^{\otimes 2}(\d\theta, \d \theta')}^{1/50}\\
  \|\why\|_t& = \pc{\int \sup_{s<t} |\wxh(s,\theta)|^{50}\phh(\d\theta)}^{1/50}\label{e:normw}\\
  \|\wxh\|_t& = \pc{\int \sup_{s<t} |\why(s,\theta)|^{50}\phh(\d\theta)}^{1/50}
\end{equs}
 so that
 \begin{equ}\label{e:supboundnorm}
   \llbracket W  \rrbracket_{\psi,t} \geq \|W\|_t
 \end{equ}
for
\begin{equs}
\|W\|_t := \|\whh\|_t \vee \|\wxh\|_t\vee \|\why\|_t\,.
\end{equs}
Note that by \aref{a:1} we have $\|\whh\|_t \leq  R < K$ uniformly in $t\geq 0$.

Furthermore, for a pair of mean-field weights $W, W'$, we define the analogous norm on differences (note the different exponent):
\begin{equs}
\|W- W'\|_t := \|\whh-\whh'\|_t \vee \|\wxh- \wxh'\|_t\vee \|\why- \why'\|_t
\end{equs}
for
\begin{equs}
 \|\whh- \whh'\|_t& :=\pc{\int \sup_{s\leq t} |\whh(s,\theta,\theta') - \whh'(s,\theta,\theta')|^{2}\phh^{
 \otimes 2
 }(\d\theta, \d \theta')}^{1/2}\\
\|\why- \why'\|_t& := \pc{\int \sup_{s<t} |\wxh(s,\theta)- \wxh'(s,\theta)|^{2}\phh(\d\theta)}^{1/2} \label{e:tnorm}\\
\|\wxh- \wxh'\|_t& := \pc{\int \sup_{s<t} |\why(s,\theta)-\why'(s,\theta)|^{2}\phh(\d\theta)}^{1/2}
\end{equs}

Throughout this section we fix an initialization $W^0$ for the mean-field weights of the network.
\begin{theorem}\label{t:eandu}
  Assume that the initialization of the MF ODEs satisfies $\llbracket W^0  \rrbracket_{\psi,0} < K$. Then under \aref{a:1} there exists a unique solution to the MF ODEs \eref{e:mfpde}.
  \end{theorem}
  Analogously to \cite{pham20global}, the proof pivots on the use of Picard's iteration. In order to apply this strategy, we define the trajectory of the weights where the RHS of the MF ODEs  is obtained by ``plugging in'' the evolution of the weights at the previous iteration with initial condition $W(0)$:
  \begin{equs}
    F_{xh}[W'](t,\theta) & := W_{xh}(0,\theta) - \int_0^t \int_X  \Delta F(W'(s), \xx)
     \int_{\Omega_{h}} \why'(\theta') \Xi_0[W'(s)](\theta;\theta', \xx) \phh(\d \theta') \nu(\d \xx) ds\\
    F_{hh}[W'](t,\theta, \theta') &:=  W_{hh}(0,\theta, \theta') - \int_0^t \int_X  \Delta F(W'(s), \xx)
     \int_{\Omega_{h}} \why'(\theta_0) \Xi_0'[W'(s)](\theta, \theta';\theta_0, \xx) \phh(\d \theta_0) \nu(\d \xx) ds\\
    F_{hy}[W'](t,\theta) &:= W_{hy}(0,\theta') - \int_0^t \int_X  \Delta F(W'(s), \xx) \sigma_h\big(  \HH(\theta; \xx, 0)\big) \nu(\d \mathbf x) ds
  \end{equs}\normalsize
  We now present a preparatory lemma, estimating the growth of
  \begin{equ}
    \|F[W']-F[W'']\|_t := {\|F_{xh}[W']-F_{xh}[W'']\|_t \vee \|F_{hh}[W']-F_{hh}[W'']\|_t\ \vee \|F_{hy}[W']-F_{hy}[W'']\|_t}
  \end{equ}
  in terms of $\|W'-W''\|_t$ in order to prove contraction of the map $F$. This result holds provided that the growth of the weight trajectories $W', W''$ is bounded in an appropriate sense.
  To state these necessary growth bounds, we introduce the key functional
  \begin{equ}\label{e:K0T}
     K_0(t) :=  K^{2L+5} (1+t^2)(1+ \llbracket W^0 \rrbracket_{\psi,0})
   \end{equ}
   that depends on a large constant $K>0$ to be chosen later.
   For any $T>0$, we also define the maximal operator
   \begin{equ}
     \max_T(W) := \sup_{s\leq T} |W_{hh}(s; \theta, \theta')|\vee| W_{hy}(s;\theta)|\vee| W_{hx}(s;\theta)|
   \end{equ}

  \begin{lemma}\label{l:9} Let \aref{a:1} hold and $\llbracket W^0\rrbracket_{0. \psi}<\infty$. For any $T>0$ and any $B>0$, consider two collections of mean-field parameters $W' = \{W'(t)\}_{t \leq T}, W'' = \{W''(t)\}_{t \leq T}$,
     assume that $ \|W'\|_T \vee \| W'' \|_T < K_0(T)$ and
    \begin{equs}
      \mathbb P(\max_T(W') > K_0(T)B )\vee \mathbb P(\max_T(W'') > K_0(T)B )\leq 2Le^{1-K_1B^2}
      \end{equs}
      for a choice of $K, K_1>0$. Then we have
    \begin{equ}
      \|F[W']-F[W'']\|_t \leq  {k_1 (1+B) \int_0^t \|W'-W''\|_s ds + k_2 e^{-k_3B^2}}
    \end{equ}
    where $k_1 = (K K_0(T))^{3L+3}$, $k_2 = T \sqrt L (K K_0(T))^{3L+3}$, $k_3 = K_1/2$.
  \end{lemma}
  Based on the definition of $K_0(t)$ from \eref{e:K0T} we define the spaces $\mathcal W_T, \mathcal W_T^0$ as the set of mean-field weight trajectories $W'$ satisfying that there exists $K>0$ such that, respectively,
  \begin{equ}
    \|W'\|_T \leq K_0(T) \end{equ}
    and
    \begin{equ}
       W'(0) = W^0,\qquad \qquad \llbracket W'\rrbracket_{T,\psi} \leq K_0(T)\,,\qquad\quad \mathbb P(\max_T(W') > K_0(T)B )\leq 2Le^{1-K_1B^2}\quad \forall B > 0
  \end{equ}
  so that $\mathcal W_T^0 \subseteq \mathcal W_T$ by \eref{e:supboundnorm}.

  .

\begin{proof}[Proof of \tref{t:eandu}]
Fix an arbitrary finite time $T>0$. By the fact that $F$ is an endomorphism in $\mathcal W_T^0$ (Lemma~8 in \cite{pham20global}), we can apply \lref{l:9} (for every $B$ with $K, K_1$ fixed) and iterating the above estimate to obtain
\begin{equs}
  \|F^{(m)}[W'] -& F^{(m)}[W'']\|_T  \leq k_1 (1+B) \int_0^T \|F^{(m-1)}[W']-F^{(m-1)}[W'']\|_{t_2} dt_2 + k_2 e^{-k_3B^2}\\
  &\leq k_1^2 (1+B)^2 \int_0^T \int_0^{t_2} \|F^{(m-2)}[W']-F^{(m-2)}[W'']\|_{t_3} dt_3 dt_2 \\
  &\qquad \qquad + k_2\sum_{\ell =1}^2\frac{(T k_1k_2 (1+B))^{\ell-1}}{\ell!} e^{-k_3B^2}\\
  &\dots\\
  &\leq k_1^m (1+B)^m \int_0^T \int_0^{t_2}\dots \int_0^{t_m} \|W'-W''\|_{t_{m+1}} dt_{m+1}\dots dt_2 \\
  &\qquad \qquad + k_2\sum_{\ell =1}^m\frac{(T k_1k_2 (1+B))^{\ell-1}}{\ell!} e^{-k_3B^2}\\
  &\leq k_1^m (1+B)^m T^m \frac 1 {m!}   \|W'-W''\|_{T}  + k_2e^{(T k_1k_2 (1+B))-k_3B^2}\\
\end{equs}
Setting $B = \sqrt m$ and choosing $W'' = F[W']$, from the above estimate we obtain
\begin{equs}
  \sum_{m=1}^\infty \|F^{(m+1)}[W'] - F^{(m)}[W']\|_T =  \sum_{m=1}^\infty \|F^{(m)}[W''] - F^{(m)}[W']\|_T < \infty
\end{equs}
showing that $F^{(m)}[W']$ is a Cauchy sequence and hence converges (the proof of completeness of the spaces $\mathcal W_T, \mathcal W_T^0$ is similar to \cite{pham20global} and is omitted). The uniqueness of the limit point is obtained by contradiction: Assume that $W', W''$ with $\|W'- W''\|_T > 0$ are fixed points of $F$. Then, again choosing $B = \sqrt m$, for every $m>0$ we have
\begin{equs}
  \|W' - W''\|_T &= \|F^{(m)}[W'] - F^{(m)}[W'']\|_T \\
  & \leq k_1^m (1+\sqrt m)^m T^m \frac 1 {m!}   \|W'-W''\|_{T}  + k_2e^{(T k_1k_2 (1+\sqrt m))-k_3m}
\end{equs}
which vanishes as $m \to \infty$ contradicting the assumption. Since the above argument goes through for every $T> 0$ we have existence and uniqueness for every $T>0$.
\end{proof}

We now introduce some more compact notation for the time differential of the mean-field weight trajectories:
\begin{equs}
   \Delta^{xh}(\xx, \theta, W'(s)) & :=  \Delta F(W'(s), \xx)
   \int \why'(\theta') \Xi_0[W'(s)](\theta;\theta', \xx) \phh(\d \theta') \\
  \Delta^{hh}(\xx, \theta, \theta', W'(s)) &:=   \Delta F(W'(s), \xx)
   \int \why'(\theta_0) \Xi_0'[W'(s)](\theta, \theta';\theta_0, \xx) \phh(\d \theta_0) \\
   & = \Delta F(W, \xx)
    \left(\sum_{i=0}^L\Gamma_i(W,\theta, \xx)\sigma_h(\HH(\theta'; \xx, i+1))\right)\\
  \Delta^{hy}(\xx, \theta,  W'(s)) &:=  \Delta F(W'(s), \xx) \sigma_h\big(  \HH(\theta; \xx, 0)\big)
\end{equs}
and defining
\begin{equ}\label{e:Dih}
  \Delta_i^{H}(\xx, \theta,  W) := \Delta F(W, \xx)\Gamma_i(W,\theta, \xx)
\end{equ}
we can write
\begin{equ}\label{e:needed1}
  \Delta^{hh}(\xx, \theta, \theta', W'(s)) = \sum_{i=0}^L \Delta_i^{H}(\xx, \theta,  W'(s))\sigma_h(\HH(\theta'; \xx, i+1))
\end{equ}

We proceed to establish the necessary a-priori growth and Lipschitz estimates to obtain the above result, defining throughout $\mathbb E_X[\cdot] := \int_X \cdot \nu(d \xx)$.

\begin{lemma}\label{l:globalstability}
  Under \aref{a:1}, given an initialization $W(0)$, a solution $W$ to the MF ODEs \eref{e:mfpde} must satisfy that for any $t>0$
  \begin{equ}
    \|W\|_t \vee \max_{i} \pc{\int_{\Omega_{h}} \sup_{s \leq t}  \mathbb E_{X}[|\Delta_i^H(\xx, \theta; W(s))|]^{50}}^{1/50} \leq K^{2L+5}(1+t^{2})(1+\|W\|_0)
  \end{equ}
  for a constant $K > 0$ large enough.
  Similarly for $\llbracket W \rrbracket_{\psi,t}$,
  there exists
  $K>0$ large enough such that $\llbracket W \rrbracket_{\psi,t} < K_0(t)$ for all $t >0$. Furthermore, for any $B>0$
  \begin{equ}
    \mathbb P(\max_t(W)\geq K_0(t) B) \leq 2 L e^{1-K_1B^2}
  \end{equ}
  for a universal constant $K_1$.
\end{lemma}

\begin{lemma}\label{l:10}
  Consider two collections of mean-field parameters $W', W'' \in \mathcal W_T$. Under \aref{a:1} for any $t<T$ and any $1\leq k \leq L$ we have
  \begin{equs}
  \pc{  \int \sup_{s\leq t} \nxx{H_\sigma[W'(s)](\theta;\xx, k ) - H_\sigma[W''(s)](\theta; \xx, k)}}^{1/2} &\leq K^{2L}  \|W'-W''\|_t\\
    \sup_{s\leq t} \mathbb E_X[ |\hat F(\xx; W'(s)) - \hat F(\xx; W''(s))|] &\leq K^{2L} K_0(T) \|W'-W''\|_t
  \end{equs}
\end{lemma}

\begin{lemma}\label{l:11}
  For a given $B>0$ consider two collections of mean-field parameters $W', W'' \in \mathcal W_T$ such that
  \begin{equs}
    \mathbb P(\max_T(W')> K_0(T)B)\leq e^{1-K_1B^2},\\ \mathbb P(\max_T(W'')> K_0(T)B)\leq e^{1-K_1B^2}
  \end{equs}
  Then under \aref{a:1}, for any $t \in [0,T]$ the following holds:
\begin{equs}
  \pc{\int \sup_{s\leq t} \nx{\Delta^{hh}(\xx, \theta, \theta', W'(s))- \Delta^{hh}(\xx, \theta, \theta', W''(s))} P_{hh}^{\otimes 2}(d\theta, d\theta')}^{1/2} \leq D(t, W', W'')
\end{equs}
where
\begin{equ}
  D(t, W', W'') := (K K_0(t))^{3L+3} \pc{(1+B) \|W'-W''\|_t + \sqrt L e^{- K_1 B^2/2}}
\end{equ}
\end{lemma}
\begin{proof}[Proof of \lref{l:9}] The proof of this lemma is performed as in Lemma 9 in \cite{pham20global}
  combining \lref{l:10} and \lref{l:11}, corresponding respectively to Lemma 10 in \cite{pham20global} and
  Lemma 11 in \cite{pham20global}.
\end{proof}

\begin{proof}[Proof of \lref{l:globalstability}]
  The proof of this lemma is analogous to the one of Lemma 6 in \cite{pham20global}.
  In the following we highlight the main differences with Lemma 6 in \cite{pham20global}, to which we refer the reader for the details of the proof. We define
  \begin{equ}
    \llbracket W_{hh} \rrbracket_{m,t} := \sqrt{\frac{50}{m}} \pc{\int \sup_{s\leq t} |W_{hh}(s,\theta,\theta')|^m \phh^{\otimes 2}(d\theta, d\theta') }^{1/m}
  \end{equ}
  and analogously for $\wxh$ and $\why$.

Starting at the output layer, we have by Cauchy-Schwarz inequality and the fact that the mean-field ODE dynamics decrease the population risk that
\begin{equ}
  \sup_{s\leq t} \mathbb E_X[\Delta F(W(s), \xx)^2] \leq \mathbb E_X[\Delta F(W(0), \xx)^2] = \sqrt{\LL[W(0)]}< K
\end{equ}
for a $K>0$ large enough.
Consequently, we can bound the RHS of the equation for $\partial_t \why$ using Cauchy-Schwarz and the boundedness of $\sigma_h<K$ as
\begin{equ}
|\partial_t \why| = |\int  \Delta F(W, \xx) \sigma_h\big(  \HH(\theta; \xx, 0)\big) \nu(\d \mathbf x)| \leq K^2
\end{equ}
so that
\begin{equ}
  \llbracket W_{hy} \rrbracket_{m,t} \leq \llbracket W_{hy} \rrbracket_{m,0} + K^2 t
\end{equ}
as desired.

The boundedness result for $\whh$ trivially holds by the truncation introduced by $\chi_R$ upon choosing $K > R$ as in \aref{a:1}, so that $\|W_{hh}(\cdot, \cdot)\|_\infty \leq R < K$ uniformly in $t$.

Finally, for $W_{xh}$ we have again by Cauchy-Schwarz
\begin{equs}
  \llbracket \wxh\rrbracket_{m,t} & \leq \llbracket \wxh \rrbracket_{m,0} + \sqrt{\frac{50}m}\pc{\int_{\Omega_{h}}t\sup_{s\leq t} |\int  \Delta F(W, \xx)
   \left(\sum_{i=0}^L\Gamma_i(W, \theta, \xx)\xx_{-i}\right)\nu(\d \mathbf x)|^m \phh(\d \theta) }^{1/m}\\
  & \leq \llbracket \wxh \rrbracket_{m,0} + \sqrt{\frac{50}m}\sqrt{\LL[W(0)]} \sum_{i=0}^L \pc{\int|\xx_{-i}|^2\nu(\d \mathbf x)}^{1/2}\pc{\int_{\Omega_{h}}t \sup_{s\leq t} \sup_{\xx}|\Gamma_i(W, \theta, \xx)|^m \phh(\d \theta) }^{1/m} \\
  & \leq \llbracket \wxh \rrbracket_{m,0} + \sqrt{\LL[W(0)]}L  \|\xx_{0}\|_{\nu} K^{2L}  t \llbracket W_{hy} \rrbracket_{m,t}  \\
  & \leq \llbracket \wxh \rrbracket_{m,0} + L  K^{2L+2}  t \llbracket W_{hy}  \rrbracket_{m,t}\\
  & \leq \llbracket \wxh \rrbracket_{m,0} + L  K^{2L+2}   t(\llbracket W_{hy} \rrbracket_{m,0} + K^2 t)
\end{equs}
where in the second upper bound we have used that $|\Gamma_i(W, \theta, \xx)| \leq K^{2L}$ uniformly in $i \in \{1, \dots, L\}$, $\xx$ and $\theta \in \Theta$ by boundedness of $\whh$ and $\sigma_h, \sigma_h'$ from \aref{a:1}. From this follows that
\begin{equ}
  \llbracket W_{xh} \rrbracket_{m,t} \leq(1+ \llbracket W \rrbracket_{m,0}) K^{2L+5} (1+t^2)\,.
\end{equ}
The probability bound follows directly from the fact that $\max_T (W)$ is $K_0(t)$ sub-Gaussian by the bounds established above.
\end{proof}

\begin{proof}[Proof of \lref{l:10}] This proof is analogous to the one of Lemma 10 in \cite{pham20global}.
  Recalling the definition of $H_\sigma$  from \eref{e:hh}
  \begin{equ}
    \HH(\theta; \xx, k) := H_{hh}(\theta;\mathbf x,k) + H_{xh}(\theta;\mathbf x_{-k})
  \end{equ}
   and slightly abusing that notation by $H_\sigma[W]$ (and similarly for $\Hhh, \Hxh$) to highlight the set of weights with respect to which the hidden state is computed, we define
  \begin{equ}
D_k^H(t):= \pc{\int_{\Omega_{h}} \sup_{s\leq t}  \mathbb E_X\pq{|H_\sigma[W'(s)](\theta;\xx, k) - H_\sigma[W''(s)](\theta;\xx, k) |^2}\phh(d\theta)}^{1/2}
    \end{equ}
    where we recall that $\mathbb E_X[\cdot] = \int_X \cdot \nu(d \xx)$.
    Proceeding to bound the above for decreasing values of $k$ we have
    \begin{equs}
      D_L^H(t) & =  \pc{\int_{\Omega_{h}} \sup_{s\leq t}   \mathbb E_X\pq{|H_{xh}[W'](\theta;\mathbf x_{-L}) -H_{xh}[W''](\theta;\mathbf x_{-L}) |^2}\phh(d\theta)}^{1/2}\\
      & = \pc{\int_{\Omega_{h}} \sup_{s\leq t}   \mathbb E_X\pq{|\wxh'(\theta;s)\mathbf x_{-L} -\wxh''(\theta;s)\mathbf x_{-L} |^2}\phh(d\theta)}^{1/2}\\
      & \leq \mathbb E_X[|\xx_{-L}|] \pc{\int_{\Omega_{h}} \sup_{s\leq t} |\wxh'(\theta;s) -\wxh''(\theta;s)|^2 \phh(d\theta)}^{1/2} \\
      & \leq K d_t(W', W'')
    \end{equs}
where we define
\begin{equs}
  d_t(W', W'') &:= \max\left\{\pc{\int_{\Omega_{h}^2} \sup_{s\leq t} |\whh'(t;\theta, \theta')- \whh''(t;\theta, \theta')|^2 \phh^{\otimes 2}(d\theta, d\theta')}^{1/2}, \right.\\
  & \qquad \qquad \qquad \pc{\int_{\Omega_{h}} \sup_{s\leq t} |\wxh'(t;\theta)- \wxh''(t;\theta)|^2 \phh(d\theta)}^{1/2},\\
  & \qquad \qquad \qquad \left. \pc{\int_{\Omega_{h}} \sup_{s\leq t} |\why'(t;\theta)- \why''(t;\theta)|^2 \phh(d\theta)}^{1/2}\right\}\label{e:smalld}
\end{equs}
For $i < L$ we have, by triangle inequality and the Lipschitz and boundedness properties on $\sigma_h$ from \aref{a:1}
\begin{equs}
D_i^H(t)  & =\left(\int_{\Omega_{h}} \sup_{s\leq t}   \mathbb E_X\left[|(\Hxh[W'](\theta;\mathbf x_{-i}) -\Hxh[W''](\theta;\mathbf x_{-i})) \right.\right.\\
& \qquad \qquad \qquad  \left. \left. + (\Hhh[W'](\theta;\xx,i)-\Hhh[W''](\theta;\xx,i))|^2\right]\phh(d\theta)\right)^{1/2}\\
  & \leq  \pc{\int_{\Omega_{h}} \sup_{s\leq t}   \mathbb E_X\pq{|\wxh'(\theta;s)\mathbf x_{-L} -\wxh''(\theta;s)\mathbf x_{-L} |^2}\phh(d\theta)}^{1/2}\\
  & \qquad \qquad \qquad+ \pc{\int_{\Omega_{h}} \sup_{s\leq t}   \mathbb E_X\pq{|\Hhh[W'(s)](\theta;\xx,i)-\Hhh[W''(s)](\theta;\xx,i))|^2}\phh(d\theta)}^{1/2}\\
  & \leq K d_t(W', W'') + K^2 D_{i+1}^H(t) + K d_t(W', W'') \\
  & \leq K^2 (d_t(W', W'')+ D_{i+1}^H(t)) 
\end{equs}
This implies that $\max_{i \in \{0, \dots, L\}}D_{i}^H(t) \leq K^{2L} d_t(W', W'')$, proving the first claim.

The second claim follows from a similar bound:
\begin{equs}
\sup_{s\leq t} \mathbb E_X\pq{\pd{ \hat F(\xx; W'(s)) - \hat F(\xx; W''(s))}} &\leq K \pc{\int_{\Omega_{h}}\sup_{s\leq t} |\why'(\theta;s)-\why''(\theta;s) |^2 \phh(d\theta)}^{1/2} + K \|\why'\|_t D_0(t)\\
& \leq K d_t(W', W'') + K K_0(t) D_0(t)
\end{equs}
yielding the desired estimate.
\end{proof}

\begin{proof}[Proof of \lref{l:11}] Again by similarity with the original reference we simply sketch this proof highlighting the differences with the present framework.

\def\tD{\tilde D}
We start the proof establishing the a priori bound
\begin{equ}
  \pc{\int_{\Omega_{h}}\sup_{s\leq t}\mathbb E_X[\Delta_{i}^H[W'(t)](\theta;\xx)]^{50} \phh(\d\theta)}^{1/50} \leq K^{2L} K_0(t)
\end{equ}
which is obtained immediately by the fact that $\pc{\int_{\Omega_{h}}\sup_{s\leq t}\mathbb E_X[\Delta_{L}^H[W'(t)](\theta;\xx)]^{50}\phh(d \theta)}^{1/50} \leq K^2 K_0(T)$ as established above and by the recursion
\begin{equ}
  \pc{\int_{\Omega_{h}}\sup_{s\leq t}\mathbb E_X[\Delta_{i}^H[W'(t)](\theta;\xx)]^{50}\phh(\d\theta)}^{1/50} \leq K^2 \pc{\int_{\Omega_{h}}\sup_{s\leq t}\mathbb E_X[\Delta_{i+1}^H[W'(t)](\theta;\xx)]^{50}\phh(\d\theta)}^{1/50}\,.
  \end{equ}
  We now consider
  \begin{equ}
    \tD_i^H(t) := \pc{\int_{\Omega_{h}} \sup_{s\leq t} \mathbb E_X\pq{|\Delta_{i}^H[W'(t)](\theta;\xx)- \Delta_{i}^H[W''(t)](\theta;\xx)|}^2 \phh(\d\theta)}^{1/2}\,.
  \end{equ}
  Starting from $i = 0$ we have
  \begin{equs}
    \tD_0^H(t)& \leq  \tD_{0}^{H,1}(t)+\tD_{0}^{H,2}(t)+\tD_{0}^{H,3}(t)
  \end{equs}
  where
  \begin{equs}
    \qquad\qquad \tD_{0}^{H,1}(t) & = K \|W_{hy}''\|_t \sup_{s\leq t}\mathbb E_X[|\hat F(\xx;W'(s)) - \hat F(\xx;W''(s))|] \leq  K_0(t)^2 K^{2L+2} d_t(W', W'')  \\
     \tD_{0}^{H,2}(t)  &= K^{2} \pc{\int_{\Omega_{h}} \sup_{s\leq t} |\why'(t;\theta)- \why''(t;\theta)|^2 \phh(\d\theta)}^{1/2} \leq K^{2}d_t(W', W'')\\
     \tD_{0}^{H,3}(t)&  =\pc{  \int_{\Omega_{h}}  \sup_{s\leq t} |\why'(s;\theta)|^2\nx{\Delta F(W''(s);\xx)\pc{\sigma_h(H_\sigma[W'(s)](\theta;\xx, 0 )) - \sigma_h(H_\sigma[W''(s)](\theta; \xx, 0))}}\Phh(\d \theta)}^{1/2}\\ & \leq K^{2L+2} K_0(t)(B d_t(W', W'') + \sqrt{\Xi(B)})
  \end{equs}
  for any $B>0$, where $\Xi(B) = 2Le^{-K_1B^2}$ and in the last bound we have separated the expectation in $\Omega_h$ using the indicator on the set $\max_t(W) > BK_0(t)$ and its complement.
  We then proceed estimating $\tD_{i}^H(t)$  from $\tD_{i-1}^H(t)$: using the boundedness of $\whh$, $\sigma_h'$, $\sigma_h$ and the Lipschitz continuity of $\sigma_h$ we have
  \begin{equ}
    \tD_{i}^H(t) \leq \tD_{i}^{H,1}(t) + \tD_{i}^{H,2}(t)+\tD_{i}^{H,3}(t)
  \end{equ}
  where  we have,
  \begin{equs}
    \tD_{i}^{H,1}(t) & = K^2 \pc{\int_{\Omega_{h}} \sup_{s\leq t} \mathbb E_X\pq{|\Delta_{i-1}^H[W'(t)](\theta;\xx)- \Delta_{i-1}^H[W''(t)](\theta;\xx)|}^2 \phh(\d\theta)}^{1/2} \leq K^2 \tD_{i-1}^{H}(t) \\
     \tD_{i}^{H,2}(t) & = K^{2L} \pc{\int_{\Omega_{h}^2} \sup_{s\leq t} |\whh'(t;\theta, \theta')- \whh''(t;\theta, \theta')|^2 \phh^{\otimes 2}(\d\theta, \d\theta')}^{1/2} \leq K^{2L}d_t(W', W'')\\
     \tD_{i}^{H,3}(t) &  =K_0(t)K^{2L + 2} \pc{  \int_{\Omega_{h}} \sup_{s\leq t} \nx{H_\sigma[W'(s)](\theta;\xx, k ) - H_\sigma[W''(s)](\theta; \xx, k)}\Phh(\d \theta)}^{1/2}\\ & \leq K_0(t) K^{4L + 3}(B d_t(W', W'') + \sqrt{\Xi(B)})\,.
  \end{equs}
  Combining the above equations results in $\max_{i \in \{0,\dots, L\}} \tD_{i}^{H}(t) \leq K_0(t)^2K^{6L + 2}((1+B)d_t(W', W'')+ \sqrt{\Xi(B)})$.

  This yields, again analogously to \cite{pham20global}, estimates on the quantities
  \begin{equs}
    \tD_{hh}^w(t) &:= \pc{\int_{\Omega_{h}} \sup_{s\leq t} \mathbb E_X[\Delta^{hh}(\xx, \theta, \theta', W'(s))-\Delta^{hh}(\xx, \theta, \theta', W''(s))]^2\phh^{\otimes 2}(\d\theta, \d\theta')}^{1/2}\\
    \tD_{xh}^w(t)& := \pc{\int_{\Omega_{h}} \sup_{s\leq t} \mathbb E_X[\Delta^{xh}(\xx, \theta, W'(s))-\Delta^{xh}(\xx, \theta,  W''(s))]^2\phh(\d\theta)}^{1/2}\\
    \tD_{hy}^w(t)& := \pc{\int_{\Omega_{h}} \sup_{s\leq t} \mathbb E_X[\Delta^{hy}(\xx, \theta, W'(s))-\Delta^{hy}(\xx, \theta,  W''(s))]^2 \phh(\d\theta)}^{1/2}
  \end{equs}
  We only perform these estimates explicitly on the first quantity, as the other ones are analogous. In this case we have, from \eref{e:needed1} by the Lipschitz continuity of $\sigma_h$ and the uniform boundedness of $\Delta_i^H$ in $L^2(\nu)$,
  \begin{equ}
    \tD_{hh}^w(t) \leq L K^2(\tD_{hh}^{w,1}(t) + \tD_{hh}^{w,2}(t))
  \end{equ}
  for
  \begin{equs}
    \tD_{hh}^{w,1}(t) & =  \max_{i \in \{1, \dots, L\}}\pc{\int_{\Omega_{h}} \sup_{s\leq t} \mathbb E_X[\Delta_i^{H}(\xx, \theta, W''(s))-\Delta_i^{H}(\xx, \theta,  W''(s))]^2\phh(\d \theta)}^{1/2}\,,\\
     \tD_{hh}^{w,2}(t) & = \max_{i \in \{1, \dots, L\}}\pc{\int_{\Omega_{h}} \sup_{s\leq t} \mathbb E_X[H_\sigma[W'(s)](\theta;\xx,  i)-H_\sigma[W''(s)](\theta;\xx,  i)]^2\phh(\d\theta)}^{1/2}\,.
  \end{equs}
  Having bounded both terms by
  \begin{equ}
    K_0(t)^2K^{6L + 2}((1+B)d_t(W', W'')+ \sqrt{\Xi(B)}) \leq K^{3L + 2}K_0(T)^{3L + 2}((1+B)d_t(W', W'')+ \sqrt{\Xi(B)})
    \end{equ}
     in the first part of this proof and in \lref{l:10} respectively concludes the argument.
\end{proof}

\section{Proof of convergence}

To prove finite time convergence for the trajectories of the large-width neural network to the corresponding mean-field limit we bound the distance
\begin{equs}
  \mathcal D_\tau(W, \mathbf W)&=   \sup_{t \in (0,\tau)}\left(\frac 1{n^2}\| \whh(t;\theta(i), \theta(j)) - \WWhh(t;i,j)\|_2 \vee \frac 1 {n} \|\wxh(t;\theta(j))-\WWxh(t; j)\|_2\right.\\
  & \qquad \qquad \qquad\qquad  \vee\left. \frac 1 {n} \|\why(t;\theta(j))-\WWhy(t; j)\|_2\right)\label{e:DT}
\end{equs}
This proof, given for completeness, adapts the steps of Proposition 25 in \cite{pham20global} to the present setting, and we therefore give it as a sketch. We require the following additional assumption
\begin{assumption}\label{a:indep}
  Let $\eta = n^{0.501}$ and consider a family of initialization laws $I$. For each $n \in \mathbb N$ the sampling rule $\bar P_n$ satisfies that $(\theta(j))_{j =1}^n \sim \bar P_n$ are $\eta$-independent, i.e. for all $1$-bounded $f~:~\Ohh \to \mathcal H$ where $\mathcal H$ is a separable Hilbert space we have
\begin{equ}
  \|\mathbb E[f(\theta(j))|\{\theta(j')~:~j'<j\}]- \mathbb E[f(\theta(j))]\|_{\mathcal H} \leq \eta\qquad \text{for all } j \in \{1,\dots,n\}
\end{equ}
  \end{assumption}

We recall the main theorem, stated together with the above assumption
  \begin{theorem}
  For any $R>0$, let Assumptions~\ref{a:01a}, \ref{a:02}, \ref{a:1} and \ref{a:indep} hold. There exist constants $c,c'>0$ such that, under \aref{a:1}, for any $\delta>0$, any $L \in \mathbb N$ and $\tau >0$, there exists $n^*\in \mathbb N$ such that for any $n>n^*$ with probability at least $1-\delta-\bar K n\exp(-\bar K n^{c'})$ we have
  \begin{equ}
    \mathcal D_\tau(W, \mathbf W) \leq \bar K n^{-c} \sqrt{\log\pc{n^2/\delta+e}}
  \end{equ}
  where $\bar K$ is a constant that depends on  $L$ and $R$.
  \end{theorem}
We will consider the evolution of the truncated version $\underline W$ of the initialization $W^0$, which is obtained by evolving according to \eref{e:mfpde} the initial condition
\begin{equs}
  \underline \wxh(0,\theta) &:= \tilde \chi_B(\wxh(0,\theta))\\
  \underline \why(0,\theta) &:= \tilde \chi_B(\why(0,\theta))\\
\end{equs}
and respectively for $\WW$
\begin{equs}
  \underline \WWxh(0,i) &:= \tilde \chi_B(\WWxh(0,\theta(i)))\\
  \underline \WWhy(0,i) &:= \tilde \chi_B(\WWhy(0,\theta(i)))\\
\end{equs}
where $\tilde \chi_B(u) = u \mathbbm{1}(|u|<B) + B \sign(u)\mathbbm{1}(|u|\geq B)$ and $ \mathbbm{1}$ is the indicator function. Note that the $\whh$ weights were not truncated as they are bounded by assumption.
Then, analogously to Proposition 27 in \cite{pham20global} one can show that with probability at least $1 -KLn\exp(-K e^{-KB^2} n^{1/52})$ we have
\begin{equ}\label{e:close}
  \|\WW- \underline \WW\|_T \vee \|W - \underline W\|_T \leq K \exp\pc{-KB^2+K^{2L+5}(1+T^{2})(1+B)}
\end{equ}
We define for any $t>0$, analogously to \eref{e:tnorm}
\begin{equs}
  \|\WW- \WW'\|_t &:= \|\WWhh-\WWhh'\|_t \vee \|\WWxh-\WWxh'\|_t\vee \|\WWhy- \WWhy'\|_t\\
\end{equs}
for
\begin{equs}
  \|\WWhh-\WWhh'\|_t &:= \pc{\frac 1 {n^2} \sum_{j_1, j_2=1}^n \sup_{s\in (0,t)} |\WWhh(s,j_1,j_2)-\WWhh'(s,j_1,j_2)|^{2}}^{1/2} \\
  \|\WWxh-\WWxh'\|_t &:= \pc{\frac 1 {n} \sum_{j_1=1}^n \sup_{s\in (0,t)} |\WWxh(s,j_1)-\WWxh'(s,j_1)|^{2}}^{1/2} \\
  \|\WWhy-\WWhy'\|_t &:= \pc{\frac 1 {n} \sum_{j_1=1}^n \sup_{s\in (0,t)} |\WWhy(s,j_1)-\WWhy'(s,j_1)|^{2}}^{1/2}
\end{equs}
Defining throughout
\begin{equ}\label{e:kt}
  K_t := K^\kappa (1+t^\kappa)
  \end{equ}
for a choice of $K, \kappa$ that can change from line to line, we proceed to show that with probability at least $1- \delta -KLn\exp(-K  n^{1/52})$
\begin{equ}\label{e:midassumption}
\DD_t( \underline W, \underline \WW) \leq
\sqrt{\frac 1 {n} \log\pc{\frac {2 T L n^2}\delta + e}} \exp(K_T(1+B))
\end{equ}
 for every choice of $\delta>0, B>0$. Combining \eref{e:close} and \eref{e:midassumption} via triangle inequality we obtain that with probability $1 -\delta - KLn\exp(-K e^{-KB^2} n^{1/52})$ we have
\begin{equs}
  \DD_t(W,  \WW) & \leq \DD_t(\underline W, \underline \WW) + \|\WW- \underline \WW\|_T + \|W - \underline W\|_T\\
  &  \leq
  \pc{\sqrt{\frac 1 {n} \log\pc{\frac {2 T L n^2}\delta + e}} + \exp(-KB^2)} \exp(K_T(1+B))
\end{equs}
and
choosing $B = c_0 \sqrt{\log n}$ for some suitable constant $c_0>0$ yields the claim of \tref{t:convergence}.

We prove the missing result \eref{e:midassumption}.
To do so, in the remainder of the section we slightly abuse notation and denote by $W, \WW$ the truncated $\underline W, \underline \WW$.
Then, for the newly defined $W, \WW$,
using that
\begin{equ}\label{e:trunc}
|\whh^0(\theta, \theta')| \leq K\,,
\end{equ}
we define the norms
\begin{equs}
  \|\WWhh\|_t &:= \pc{\frac 1 {n^2} \sum_{j_1, j_2=1}^n \sup_{s\in (0,t)} |\WWhh(s,j_1,j_2)|^{50}}^{1/50} \\
  \|\WWxh\|_t &:= \pc{\frac 1 {n} \sum_{j_1=1}^n \sup_{s\in (0,t)} |\WWxh(s,j_1)|^{50}}^{1/50} \label{e:normW}\\
  \|\WWhy\|_t &:= \pc{\frac 1 {n} \sum_{j_1=1}^n \sup_{s\in (0,t)} |\WWhy(s,j_1)|^{50}}^{1/50}
\end{equs}
and for a given realization of the sampling $\bar P_n$,
\begin{equs}
  \|W\|_{\text{samp},t}
  & = \pc{\frac 1 {n} \sum_{i=1}^n \sup_{s<t} |\why(s,\theta(i))|^{50}}^{1/50} \vee\pc{\frac 1 n \sum_{i=1}^n\int \sup_{s<t} |\whh(s,\theta(i),\theta')|^{50}\phh( \d \theta')}^{1/50}\\&
  \vee \pc{\frac 1 n \sum_{i=1}^n\int \sup_{s<t} |\whh(s,\theta',\theta(i))|^{50}\phh( \d \theta')}^{1/50}\\&
  \vee\pc{\frac 1 {n^2} \sum_{i,j=1}^n \sup_{s<t} |\whh(s,\theta(i),\theta(j))|^{50}}^{1/50}
  \vee\pc{\frac 1 {n} \sum_{i=1}^n \sup_{s<t} |\wxh(s,\theta(i))|^{50}}^{1/50}
\end{equs}
 Then, analogously to \lref{l:globalstability} and in turn Lemma 30 in \cite{pham20global} one can show that for each $\| W\|_0$, $\|W\|_{\text{samp},0}$ there exists $\kappa$ such that, respectively,
$
  \|\mathbf W\|_t \leq K_t
$ and $\|W\|_{\text{samp},t}\leq K_t $.

Further, we define $\mathcal E$ as the event
\begin{equ}
  \mathcal E := \{\|\WW\|_0 \vee \|W\|_{\text{samp},0} < K\}
\end{equ}
which holds with a probability of at least $1-KLn\exp(-K n^{1/52})$ by Lemma 29 in \cite{pham20global}
since by assumption $\|W\|_0< K$. This directly implies that $  \|\WW\|_t \vee \|W\|_{\text{samp},t} < K_t$ by Lemma 30 in \cite{pham20global}.

We start by decomposing, for any $\xi > 0$,
\begin{equs}
  \DD_t(W, \WW) &\leq
  K \int_0^t (\dxh^w(\lfloor s/\xi\rfloor\xi) + \dhh^w(\lfloor s/\xi\rfloor\xi) +  \dhy^w(\lfloor s/\xi\rfloor\xi)  ) \d s \label{e:firstlineDD}\\
  &\qquad + Kt \sup_{s \in (0, T-\xi)}\sup_{\xi'\in (0,\xi)}\max_{V \in \{W, \WW\}}\pc{\dxh^\xi[V](s,\xi')\vee \dhy^\xi[V](s,\xi') \vee \dhh^\xi[V](s,\xi')}
\end{equs}
where
\begin{equs}
  \dhh^w(t) &:= \pc {\frac 1 {n^2} \sum_{j,k =1}^n |\partial_t \WWhh(t;j,k) - \partial_t \whh(t;\theta(j);\theta(k))|^2}^{1/2}\\
      \dxh^w(t) &:= \pc {\frac 1 {n} \sum_{j =1}^n |\partial_t \WWxh(t;j) - \partial_t \wxh(t;\theta(j))|^2}^{1/2}\\
        \dhy^w(t) &:= \pc {\frac 1 {n} \sum_{j=1}^n |\partial_t \WWhy(t;j) - \partial_t \why(t;\theta(j))|^2}^{1/2}
      \end{equs}
      and
      \begin{equs}
        \dhh^\xi[\WW](t, \xi') &:= \pc {\frac 1 {n^2} \sum_{j,k =1}^n |\partial_t \WWhh(t;j,k) - \partial_t \WWhh(t+\xi';j,k)|^2}^{1/2}\\
            \dxh^\xi[\WW](t, \xi') &:= \pc {\frac 1 {n} \sum_{j =1}^n |\partial_t \WWxh(t;j) - \partial_t \WWxh(t+\xi';j)|^2}^{1/2}\\
              \dhy^\xi[\WW](t, \xi') &:= \pc {\frac 1 {n} \sum_{j=1}^n |\partial_t \WWhy(t;j) - \partial_t \WWhy(t+\xi';j)|^2}^{1/2}\\
              \dhh^\xi[W](t, \xi') &:= \pc {\frac 1 {n^2} \sum_{j,k =1}^n |\partial_t \whh(t;\theta(j),\theta(k)) - \partial_t \whh(t+\xi';\theta(j),\theta(k))|^2}^{1/2}\\
                  \dxh^\xi[W](t, \xi') &:= \pc {\frac 1 {n} \sum_{j =1}^n |\partial_t \wxh(t;\theta(j)) - \partial_t \wxh(t+\xi';\theta(j))|^2}^{1/2}\\
                    \dhy^\xi[W](t, \xi') &:= \pc {\frac 1 {n} \sum_{j=1}^n |\partial_t \why(t;\theta(j)) - \partial_t \why(t+\xi';\theta(j))|^2}^{1/2}
\end{equs}
The following lemma, proven at the end of the section, bounds the error resulting from the time-disctretization in $\xi$:
\begin{lemma}\label{l:xibound} For any $\xi \in [0,T]$ we have that almost surely on the event $\mathcal E$
  \begin{equ}
    \sup_{s \in (0, T-\xi)}\sup_{\xi'\in (0,\xi)}\max_{V \in \{W, \WW\}}\pc{\dxh^\xi[V](s,\xi')\vee \dhy^\xi[V](s,\xi') \vee \dhh^\xi[V](s,\xi')} \leq K_T (1+B)\xi
  \end{equ}
\end{lemma}

We now proceed to bound the terms on the first line of \eref{e:firstlineDD}. To do so we  define for $\ell \in \{1,\dots, L\}$ \def\fs{M_\sigma}
\begin{equs}
  \ghh^\ell(t) &:= \pc{\frac 1 n \sum_{j=1}^n \mathbb E_X\pq{|\dehh^{\hh}(\xx,j,\WW(t),\ell) - \dehh^{H}(\xx,j,W(t),\ell)|}^2}^{1/2}\\
  \fhh^\ell(t) &:= \pc{\frac 1 n \sum_{j=1}^n \mathbb E_X\pq{|{\hhhh}(\xx,j,\WW(t),\ell) - \Hhh(\xx,j,W(t),\ell)|}^2}^{1/2}\\
  \fxh^\ell(t) &:= \pc{\frac 1 n \sum_{j=1}^n \mathbb E_X\pq{|{\hhxh}(\xx,j,\WW(t),\ell) - \Hxh(\xx,j,W(t),\ell)|}^2}^{1/2}\\
  \fs^\ell(t) &:= \pc{\frac 1 n \sum_{j=1}^n \mathbb E_X\pq{|{\mathbf H_\sigma}(\xx,j,\WW(t),\ell) - H_\sigma(\xx,j,W(t),\ell)|}^2}^{1/2}
\end{equs}
where
\begin{equs}
  \Hxh(\xx,j,W,\ell) &:= W_{xh}(\theta(j))\cdot \xx_{-\ell} \\
  \hhxh(\xx,j,\mathbf W,\ell) &:= \mathbf W_j \cdot \xx_{-\ell}\\
  \Hhh(\xx,j,W,\ell) &:= \Hhh[W](\theta(j),\xx , \ell)\\
  \hhhh(\xx,j,\mathbf W,\ell) &:= \hhhh[\mathbf W]( \xx, \ell)_j\\
  \HH(\xx,j,W,\ell) &:= \Hhh(\xx,\theta(j),W,\ell) + \Hxh(\xx,\theta(j),W,\ell)\\
  \mathbf H_\sigma(\xx,j,\mathbf W,\ell) &:= \hhhh(\xx,j,\mathbf W,\ell) + \hhxh(\xx,j,\mathbf W,\ell)\\
  \dehh^{H}(\xx,j,W,\ell)& := \Delta_\ell^{H}(\xx,\theta(j),W)\\
\dehh^{\hh}(\xx,j,\WW,\ell) &:= \Delta_\ell^{\mathbf H}(\xx,j,\WW)
\end{equs}
for $\Delta_\ell^{H}$ defined in \eref{e:Dih} and
\begin{equ}
  \Delta_\ell^{\mathbf H}(\xx, j,  \WW) := \Delta F(\WW, \xx)\boldsymbol \Gamma_\ell(\WW,j, \xx)
\end{equ}
for
\begin{equs}\label{e:gammai}
  \boldsymbol \Gamma_\ell(\WW, j,\xx) &:= \frac 1 n \sum_{j_0=1}^n \WWhy({j_0}) \sigma_h'(\mathbf H_\sigma(\xx,j_{0},\mathbf W,0)) \frac 1 n \sum_{j_1=1}^n \WWhh({j_0}, j_1) \sigma_h'(\mathbf H_\sigma(\xx,j_{1},\mathbf W,1))\\
  &\qquad\qquad
  \dots\frac 1 n \sum_{j_{\ell-1}=1}^n \WWhh({j_{\ell-1}}, j) \sigma_h'(\mathbf H_\sigma(\xx,j,\mathbf W,\ell))
\end{equs}
Further defining
\begin{equs}
  \ds^{w,\ell}(t) &:= \left(\frac 1 n \sum_{j=1}^n \mathbb E_X\pq{1+|\dehh^{\hh}(\xx,j,\WW(t),\ell)|^2 + |\dehh^{H}(\xx,j,W(t),\ell)|^2}\right)^{1/2}\\
  &\qquad \cdot\left[\left( \frac 1 n \sum_{j=1}^n \mathbb E_X\pq{|{\hhhh}(\xx,j,\WW(t),\ell) - \Hhh(\xx,j,W(t),\ell)|}^2\right)^{1/2}\right.\\
  &\qquad \qquad + \left.\left( \frac 1 n \sum_{j=1}^n \mathbb E_X\pq{|{\hhxh}(\xx,j,\WW(t),\ell) - \Hxh(\xx,j,W(t),\ell)|}^2\right)^{1/2}\right]
\end{equs}
we have that on the event $\mathcal E$ by Lemma 30 in \cite{pham20global}
$\ds^{w,\ell}(t) \leq K_T \fs^\ell(t)$, so that on the same event we have
\begin{equ}
  \dhh^{w}(t)\leq  K_T \sum_{\ell = 0}^{L-1} \pc{\fs^{\ell+1}(t) + \ghh^{\ell+1}(t) }
\end{equ}
and analogously
\begin{equs}
  \dhy^{w}(t)&\leq  K_T  \pc{ \fs^{0}(t) + \ghh^{0}(t) }\\
  \dxh^w(t)& \leq  K_T \sum_{\ell = 0}^L  \ghh^{\ell}(t)
\end{equs}
Combining these bounds with \lref{l:xibound} we obtain that on the event $\mathcal E$
\begin{equ}\label{e:convergence2}
  \DD_t(W, \WW) \leq K_T \pc{\int_0^t\sum_{\ell = 0}^{L}\pc{ \fs^\ell(s) + 2\ghh^\ell(s)} + (1+B)\xi \,\d s}
\end{equ}

We further proceed to bound $\fs^\ell(s) + 2\ghh^\ell(s)$ in terms of $\DD_t(W, \WW)$ with high probability as follows:
\begin{lemma}\label{l:Fbound}
  For any sequence $\{\gamma_j\}_{j=0}^L$ with $\gamma_j > 0$, for all $k \in \{0, \dots ,L \}$ and $t \in (0, T)$ the event $\mathcal E_{t,k}^\hh$ where
  \begin{equ}
    \fs^\ell(t) \leq K_T^{L-\ell+1} \pc{\DD_t(W, \WW) + (1+B)\sum_{j=\ell}^{L-1} \gamma_{j}}\qquad \text{holds for all } \ell \in \{k, k+1, \dots ,L\}
    \end{equ}
    has probability
    \begin{equ}
      \mathbb P(\mathcal E_{t,k}^\hh|\mathcal E) \geq 1- \sum_{j =k}^{L } \frac n {\gamma_j} \exp(-n\gamma_j^2/K_T)\,.
    \end{equ}
\end{lemma}
\begin{lemma}\label{l:Gbound}
  For any sequence $\{\beta_j\}_{j =1}^{L}$ with $\beta_j > 0$, for all $k \in \{0, \dots ,L \}$ and $t \in (0, T)$ the event $\mathcal E_{t,k}^\Delta$ where
  \begin{equ}
    \ghh^\ell(t) \leq K_T^{L+\ell+1} \pc{(1+B)\DD_t(W, \WW) + (1+B^2)\pc{\sum_{j=0}^{L-1} \gamma_{j} + \sum_{j=1}^{\ell} \beta_{j}}}\qquad \text{holds for all } \ell \in \{0, 2,\dots,k\}
    \end{equ}
satisfies
    \begin{equ}
      \mathbb P(\mathcal E_{t,k}^\Delta \cap \mathcal E_{t,0}^\hh|\mathcal E) \geq \mathbb P(\mathcal E_{t,0}^\hh|\mathcal E) - \sum_{j =1}^k \frac n {\beta_j} \exp(-n\beta_j^2/K_T)\,.
    \end{equ}
\end{lemma}

Combining the above lemmas with \eref{e:convergence2} and \lref{l:xibound} yields that for every $B>0$ we have
\begin{equ}\label{e:convergence3}
  \DD_t(W, \WW) \leq K_T^{2L} \int_0^t \pc{(1+B)\DD_s(W, \WW) + (1+B^2)\pc{\sum_{j=1}^L \gamma_{j+1} + \sum_{j=1}^L \beta_{j+1}} + (1+
  B)\xi}\d s
\end{equ}
with probability at least
\begin{equ}
  1- \frac T \xi\pc{\sum_{j =1}^{L-1} \frac n {\gamma_j} \exp(-n\gamma_j^2/K_T) - \sum_{j =2}^L \frac n {\beta_j} \exp(-n\beta_j^2/K_T)} - KLn\exp(-Kn^{1/52})
\end{equ}
The proof is concluded applying Gronwall's lemma with
\begin{equ}
  \gamma_j = \beta_j := \sqrt{\frac 1 {K_Tn} \log\pc{\frac {2 T L n^2}\delta + e}}\qquad \text{and} \qquad \xi = \frac 1 {\sqrt n}
\end{equ}
which gives, for all $t<T$ and for all $\delta>0$, $B>0$
\begin{equs}
  \DD_t(W, \WW)& \leq K_T \pc{(1+B^2)\pc{\sum_{j=1}^L \gamma_{j+1} + \sum_{j=1}^L \beta_{j+1}} +  (1+B)\xi}\exp(K_T(1+B)T)\\
  & \leq K_T \pc{\sum_{j=1}^L \gamma_{j+1} + \sum_{j=1}^L \beta_{j+1}+  \xi}\exp(K_T(1+B))\\
  & \leq K_T 2L\sqrt{\frac 1 {K_Tn} \log\pc{\frac {2 T L n^2}\delta + e}} \exp(K_T(1+B))
\end{equs}
with probability
\begin{equs}
  \mathbb P(\mathcal E \cap \mathcal E_{T,0}^\hh \cap \mathcal E_{T,L}^\Delta) &= \mathbb P( \mathcal E_{T,0}^\hh \cap \mathcal E_{T,L}^\Delta|\mathcal E) \mathbb P(\mathcal E)\\
  & > 1- 2L \sqrt n T\pc{ \frac n {\gamma_1} \exp(-n\gamma_1^2/K_T) }  - KLn\exp(-Kn^{1/52})\\
  &> 1- \delta  - KLn\exp(-Kn^{1/52})
\end{equs}
thereby proving \eref{e:midassumption}, as desired.\qed

We now proceed with the verification of the claims that led to this conclusion. We limit ourselves to checking \lref{l:xibound} and \lref{l:Fbound} as the proof of \lref{l:Gbound} is analogous.

\begin{proof}[Proof of \lref{l:Fbound}] We show the claim by induction on the depth of the unrolled network. Starting from $\fs^{L}$ we have that with probability $1$
  \begin{equs}
    \Biggl(\frac 1 n \sum_{j=1}^n \mathbb E_X[|&\hh_\sigma(\xx,j,\WW(t),L)  - H_\sigma(\xx, \theta(j), W(t), L)|]^2\Biggr)^{1/2} \\
    &= \pc{\frac 1 n \sum_{j=1}^n\mathbb E_X[|\WWxh(t,j) \xx_{-L} - \wxh(t,\theta (j)) \xx_{-L}|]^2}^{1/2} \leq K \DD_t(W, \WW)\label{e:hxh}
  \end{equs}
  In other words, the base case holds with probability $\mathbb P(\mathcal E_{t,1}^\hh) = 1$.

  We now assume that the claim holds for $\fs^{\ell+1}$ and prove it for $\fs^\ell$. The proof for $\fhh^\ell$, $\fxh^\ell$ is analogous. To do so we decompose $\fs^{\ell}$ in two parts: the first measures the distance between a randomly sampled, finite set of weights evolving according to $W(t)$ and $\WW(t)$, while the second compares the approximation obtained by taking a finite sample from $W(t)$ and the expectation \abbr{wrt} $\Phh$ on $W(t)$.
More specifically we decompose
\begin{equs}
  |\hhhh(x,i,\WW(t),\ell) & - \Hhh(x, \theta(i), \whh(t), \ell)| = \\
  &= \Big|\frac 1 n \sum_{j=1}^n \WW(t,i,j)\sigma_h( \hhhh(\xx,j, \WW(t), \ell+1)+\hhxh(j, \xx, \WW(t)))\\
   & \qquad - \int W(t,\theta(i), \theta') \sigma_h( \Hhh(\xx,\theta', W(t), \ell+1)+\Hxh[W(t)](\theta'; \xx)) \phh(d\theta')\Big|\\
   & = Q_{1,\ell}(t;i) + Q_{2,\ell}(t;i)
\end{equs}
where
\begin{equs}
  Q_{1,\ell}(t;i) &=\frac 1 n \sum_{j = 1}^n \Big | \WW_{hh}(t,i,j)\sigma_h( \hhhh(\xx,j, \WW(t), \ell+1)+\hhxh(j, \xx, \WW(t))) \\
  &\qquad \qquad - W_{hh}(t,\theta(i),\theta(j)) \sigma_h( \Hhh(\xx,\theta(j), W(t), \ell+1) + \Hxh[W(t)](\theta(j), \xx ))\Big|\\
  Q_{2,\ell}(t;i) &= \Bigg| \frac 1 n \sum_{j=1}^nW_{hh}(t,\theta(i),\theta(j)) \sigma_h( \Hhh(\xx,\theta(j), W(t), \ell+1)+\Hxh[W(t)](\theta(j), \xx ))\\
  & \qquad \qquad - \int W_{hh}(t,\theta(i),\theta') \sigma_h( \Hhh(\xx,\theta', W(t), \ell+1) + \Hxh[W(t)](\theta', \xx)) \phh(d\theta')\Bigg|
\end{equs}
and we can bound
\begin{equ}
  \fs^{\ell}(t) \leq \pc{\frac 1 n \sum_{i=1}^n\mathbb E_X[|Q_{1,\ell}(t;i)| + |Q_{2,\ell}(t;i)|]^{2}}^{1/2} + \pc{\frac 1 n \sum_{j=1}^n\mathbb E_X[|\WWxh(t,j) \xx_{-\ell} - \wxh(t,\theta (j)) \xx_{-\ell}|]^2}^{1/2}
\end{equ}
The first term is then bounded by
\begin{equs}
  \mathbb E_X\pc{|Q_{1,\ell}(t;i)|}^{2}& \leq \frac K n \sum_{j=1}^n\pc{1+|\WW_{hh}(t,j,i)|^2 +|W_{hh}(t,\theta(j),\theta(i))|^2}\\
  & \qquad \qquad \cdot \frac 1 n \sum_{j=1}^n  \mathbb E_X\pc{ \pd{\hh_\sigma(\xx,j, \WW(t), \ell+1) - H_\sigma(\xx,\theta(j), W(t), \ell+1)}}^2\\
    & + \frac K n \sum_{j=1}^n |\WW_{hh}(t,i,j)-W_{hh}(t,\theta(i),\theta(j))|^2
\end{equs}
and therefore,  under the event $\mathcal E_{t,\ell+1}^\hh$ and $\mathcal E$ we have
\begin{equ}
\pc{\frac 1 n \sum_{i=1}^n\mathbb E_X\pq{|Q_{1,\ell}(t;i)|}^{2}}^{1/2} \leq K_T \fs^{\ell+1}(t) + K \DD_t(W, \WW)
\end{equ}
We proceed to bound $Q_{2,\ell}(t)$. Defining
\begin{equ}
  Z_\ell^H(t,\theta, \theta') = \whh(t, \theta, \theta' ) \sigma_h(\Hhh(\xx, \theta',\ell+1)+ \wxh(t, \theta) \xx_{- (\ell+1)})
\end{equ}
Using independence  of $\theta, \theta'$, we have that the conditional expectation \abbr{wrt} $\phh$ is trivial
\begin{equ}
  \mathbb E_{\phh}[Z_\ell^H(t,\theta(i), \theta(j))| \theta(i)] = \mathbb E_{ \phh}[Z_\ell^H(t,\theta(j),\theta')]
\end{equ}
and we have that, for almost every $\xx$ almost surely by assumption
\begin{equ}
  Z_\ell^H(t,\theta(i), \theta(j)) \leq K_T(1+B)
\end{equ}
Then, by Lemma 28 in \cite{pham20global},
since $\gamma_\ell \geq 0$ we have that 
\begin{equ}
  \mathbb P\pc{\mathbb E_X[ Q_{2,\ell}(t) ]\geq K_T (1+B)\gamma_\ell} \leq \frac 1 {\gamma_\ell} \exp(-n \gamma_\ell^2/K_T)\,.
\end{equ}
The proof is concluded by combinging the bound on $H_{hh}$ with the one on $H_{xh}$ to yield an analogous one on $H_\sigma$ and taking an union bound over $i \in \{1,\dots, n\}$, resulting in the fact that on the events $\mathcal E$ and $\mathcal E_{t,\ell+1}^\hh$
\begin{equs}
  \fs^{\ell}(t) \geq K_T \fs^{\ell+1} (t) + 2 K \DD_t(W, \WW) + K_T (1+B) \gamma_\ell \geq K_T^{L-\ell+1} \pc{ \DD_t(W, \WW) + (1+B)\sum_{k=\ell}^{L-1}\gamma_k}
\end{equs}
with probability at most $(n/\gamma_\ell) \exp(-n\gamma_\ell^2/K_T)$.
Therefore we get by union bound
\begin{equ}
  \mathbb P((\mathcal E_{t,\ell}^\hh)^c| \mathcal E) \leq \mathbb P((\mathcal E_{t,\ell+1}^\hh)^c| \mathcal E) +(n/\gamma_\ell) \exp(-n\gamma_\ell^2/K_T)  \leq \sum_{k=\ell}^{L-1} \frac n {\gamma_{k+1}} \exp(-n\gamma_{k+1}^2/K_T)
\end{equ}
proving the desired claim.
\end{proof}

\begin{proof}[Proof of \lref{l:xibound}]
We again only sketch this proof for the term $\dhh^\xi[W](t, \xi')$ as the other cases follow analogously.

We see that since $\|W\|_0, \|W\|_{\text{samp},t} \leq K$ on the event $\mathcal E$, we have
\begin{equs}
\pc{\frac 1 {n^2}\sum_{i,j=1}^n \sup_{s\in (0,t)}\pd{\partial_t W_{hh}(s,\theta(i), \theta(j))}^{50}}^{1/50}\leq K + K \pc{\frac 1 {n}\sum_{j=1}^n \sup_{s\in (0,t)}\mathbb E_x[|\Delta_i^H(\xx,\theta(j),W(s))|]^{50}}^{1/50}\leq K_T
\end{equs}
for any $t\leq T$. Consequently we have
\begin{equ}
  \pc{\frac 1 {n^2}\sum_{i,j=1}^n \sup_{s\in (0,T-\xi)}\sup_{\xi'\in (0,\xi)}\pc{W_{hh}(s+\xi',\theta(i), \theta(j))-W_{hh}(s,\theta(i), \theta(j))}^{2}}^{1/2} \leq K_T \xi
\end{equ}
The desired bound results from the application of an adapted version of \lref{l:11} to the paths $W'(t) := W_{hh}(t,\theta(i), \theta(j))$, $W''(t) := W_{hh}(t+\xi,\theta(i), \theta(j))$ replacing $e^{-K_1B^2} \to 0$ by the assumed trunctaion of $W$. This yields almost surely on $\mathcal E$
\begin{equ}
  \sup_{t\in (0,T-\xi)}\sup_{\xi'\in (0,\xi)} \dhh^\xi(t,\xi') \leq K_T (1+B) \|W'-W''\|_{T-\xi} \leq K_T (1+B) \xi
\end{equ}
as desired. Analogous bounds on $\dhh^\xi[\WW], \dxh^\xi[W], \dxh^\xi[\WW],\dhy^\xi[W], \dhy^\xi[\WW]$ prove the lemma.
\end{proof}

\section{Global optimality}

Recall the definition of the preactivation between the first and the second layer:
\begin{equ}
\Hhh(\theta;\mathbf x,L) = \int \whh(\theta, \theta') \sigma_h(W_{xh}(\theta')\cdot \mathbf x_{-L}) \Phh(\d \theta')
\end{equ}
and define recursively the corresponding $\ell$-preactivaton for $\ell\leq L-1$
\begin{equ}
\Hhh(\theta;\mathbf x,\ell) := \int \whh(\theta, \theta') \sigma_h(\Hhh(\theta';\mathbf x, \ell+1)+\wxh(\theta')\mathbf x_{-(\ell+1)})\,\Phh(\d \theta')
\end{equ}

For notational convenience, we define $\nu_{L,\ell}:= \Pi_\#^{(-L, -L+\ell)} \nu$, where $\Pi^{(a,b)}$ is projection on coordinates ranging from $a$ to $b$.

\subsection{Expressivity at initialization}

In this section we prove our main expressivity result. Defining throughout $\Theta := \supp (\phh)$ we we state the result as follows:
\begin{proposition}\label{p:spanning}
  Fix $L>0$, for any $t>0$ let $W = W(t)$ satisfy \aref{a:1}b) and c), let $\sigma_h$ satisfy \aref{a:1}a) and let Assumptions~\ref{a:01a} and \ref{a:02} hold. Then
  \begin{equ}
    \span \pg{\sigma_h(H_{hh}(\theta;\mathbf x,0)+\wxh(\theta)\mathbf x_{0})~:~\theta \in \Theta} = L^2(\nu_{L,0})
    \end{equ}
  \end{proposition}
  The above result can readily be rephrased in the following, more explicit form:
  \begin{corollary} Under the conditions of Proposition~\ref{p:spanning} above, the map
\begin{equ}
\hat F(W; \mathbf x ) = \int \why( \theta) \sigma_h(H_{hh}(\theta;\mathbf x,0)+\wxh(\theta)\mathbf x_{0}) \phh(\d\theta)
\end{equ}
intended as a functional of $W_{hy}\in L^2(\phh)$ is dense in the space $L^2(\nu_{L,0})$.
\end{corollary}
\pref{p:spanning} above proves that the network can express any function in $L^2(\nu_{L,0})$ provided that the support of the weights $W(t)$ is sufficiently varied as codified in  \aref{a:1}b). We will show in the next subsection that this condition, if satisfied at initialization,  is also satisfied at every finite time throughout the dynamics.

To prove the above result we first state the following
\begin{lemma}\label{l:above}
  Let $\mathcal X \subseteq \mathbb R^{d'}$ for $d' \in \mathbb N$ and let $\mu$ be a probability measure on $\mathcal X$. Assume that $\sigma_h~:~\mathbb R \to \mathbb R$ satisfies \aref{a:1}, 
  that the set $\Phi := \{\phi_\theta~:~\theta \in \Theta\} \subseteq L^2(\mu)\cap L^{\infty}(\mu)$ is star-shaped at $0 \in L^2(\mu)$ and $\span \pg{\phi_\theta~:~\theta \in \Theta}$ is dense in $L^2(\mu)$, then so is
  $\span \pg{\sigma_h(\phi_\theta) ~:~\theta \in \Theta }$.
\end{lemma}

\begin{proof}[Proof of \lref{l:above}]
  Assume towards a contradiction that there exists $f^* \in L^2(\mu)$ such that for any sequence $\{\phi_n\}_n$ with $\phi_n \in L^2(\mu)$ we have $\int f^* \sigma_h(\phi_n) \mu(d x) = 0$ for all $n \in \mathbb N$.   By the spanning assumption there exists $\phi^* \in \Phi$ with  $ \delta^*:=|\int \phi^* f^* \mu(dx)|>0$.  We now consider the sequence of functions $\phi_n = \pc{n\|\phi^*\|_\infty}^{-1} \phi^*$. By assumption on the star-like structure of $\Phi$, $\phi_n \in \Phi$ for all $n>0$.
  The result of the lemma follows by the Taylor expansion of $\sigma_h$ around the point $0$:
  \begin{equ}
    \sigma_h(\phi_\theta(x) ) = 0 + \sigma_h'(0) \phi_\theta(x) + R[\phi_\theta](x)
  \end{equ}
  where, denoting by $\mathcal B_\rho^\infty(0)$ the ball of radius $\rho$ in the $L^\infty(\mu)$ norm around $0$, there exists a constant $C>0$ such that the remainder term satisfies $|R[\phi](x)| < C \phi(x)^2$ uniformly in $\phi \in \mathcal B_\rho^\infty(0)$ and $x$ for $\rho$ small enough. Then, along the sequence $\{\phi_n\}_n$ we have
  \begin{equ}\label{e:triangle}
  \pd{  \int  f^*(x) \sigma_h( \phi_n(x) ) \mu(d x)} \geq \pd{\sigma_h'(0) \int f^*(x) \phi_n(x) \mu(dx)} -  \pd{\int f^*(x)R[ \phi_n](x) \mu(dx)}
  \end{equ}
  We notice that for $n \in \mathbb N$ sufficiently large we have,
  \begin{equs}
    \pd{\sigma_h'(0) \int f^*(x) \phi_n(x) \mu(dx)} & = \frac {\sigma_h'(0)} {n\|\phi^*\|_\infty} \delta^*\\
    \pd{ \int f^*(x)R[\phi_n](x) \mu(dx)} \leq \|f^*\|_2 \|R[\phi_n](x)\|_2 &\leq \frac {C\|(\phi^*)^2\|_2} {n^2\|\phi^*\|_\infty^2}  \|f^*\|_2 \leq \frac 1 2 \frac {\sigma_h'(0)} {n\|\phi^*\|_\infty} \delta^*
  \end{equs}
  so that the first term in the expansion dominates the second. Combining this with \eref{e:triangle} implies that there exists $n$ large enough such that
  \begin{equ}
    \pd{\int  f^*(x) \sigma_h( \phi_n(x)) \mu(d x)} >\frac 1 2 \frac {\sigma_h'(0)} {n\|\phi^*\|_\infty} \delta^* > 0
  \end{equ}
  contradicting the fact that $\int f^* \phi_n \mu (d x) = 0$ for all $n \in \mathbb N$.
\end{proof}

 \begin{proof}[Proof of \pref{p:spanning}]
  We want to show that, for any $\ell \in \{0,\dots, L\}$,
    \begin{equ}\label{e:induction}
      \span \{\sigma_h(\Hhh (\theta; \xx,\ell) + \wxh \xx_{-\ell})~:~\theta \in \supp(\phh)\} = L^2(\nu_{-L,-\ell})
    \end{equ}
   By the deterministic nature of the dynamical system $T$ the measure $\nu_{-L,-\ell}$ can be written, in the sense of distributions, as
  \begin{equs}
    \nu_{-L,-\ell}(\d \xx) &= \nu(\d\xx_{-\ell}|\xx_{-\ell})\dots \nu(\d\xx_{-L+1}|\xx_{-L})\nu_0(\d \xx_{-L})\\
    & = \nu_0(\d \xx_{-L}) \prod_{j=1}^\ell\delta(\xx_{L-j} - T^j(\xx_{-L}))
    \end{equs}
  so that, integrating on $\xx_{-L}, \dots, \xx_{-\ell}$ we write $$\Hhh(\theta; x,\ell) := \Hhh(\theta;(x, T(x),\dots, T^{L-\ell+1}(x)),\ell)\,.$$  Then, condition \eref{e:induction} can be written as
  \begin{equ}
      \span \{\sigma_h(\Hhh (\theta; x,\ell) + \wxh T^{L-\ell}(x))~:~\theta \in \supp(\phh)\} = L^2(\nu_{0})
    \end{equ}
  which, since $\{0\} \times L_R^\infty(\phh) \subset \supp\{\wxh(t;\theta), \whh(t;\theta,\cdot)~:~\theta \in \Theta\}$ by \lref{l:support} follows if
  \begin{equ}\label{e:toshow}
    \span_\theta \pq{\int \sigma\pc{\whh(\theta, \theta')\sigma\pc{\whh(\theta', \theta'') \sigma \pc{\dots \sigma_h(\wxh(\theta^{(L)})x)}}} \Phh^{\otimes L}(d \theta',\dots, d \theta^{(L)})} = L^2(\nu_{0})
  \end{equ}
We prove \eref{e:toshow} by induction on the depth of the unrolled network.

{\bf Base case $\ell=L$:} In this case we simply need to show that $\span\{\sigma_h(\wxh(\theta) \xx_{-L})~:~\theta  \in \Theta\}$ is dense in $L^2(\nu_{-L}) = L^2(\nu_0)$. This, however, is immediately true by the global approximation property \aref{a:1}b).

{\bf Induction step  $\ell\to \ell-1$:}
By \lref{l:above} it is sufficient to show that
\begin{equ}
  \Hhh'(\theta; x, \ell) := \int \whh(\theta, \theta')\sigma_h\pc{\whh(\theta', \theta'') \sigma_h \pc{\dots \sigma_h(\wxh(\theta^{(L)})x)}} \Phh^{\otimes (L-\ell+1)}(d \theta^{(\ell)},\dots, d \theta^{(L)})
  \end{equ}
spans the desired space.
  This claim is true if having
\begin{equ}
 \int\bar g(x)  \Hhh'(\theta; x, \ell-1)  \nu_{0}(\d  x) =  \int \bar g(x) \int \whh(\theta, \theta') \sigma_h(\Hhh(\theta'; x,\ell)
 ) \Phh(\d \theta') \nu_{0}(\d x)=0
\end{equ}
for almost all $\theta \in \Ohh$ implies that the function $\bar g~:~\mathbb R^{\din} \to \mathbb R$ must satisfy $\bar g(x) \equiv 0$.
Using \lref{l:support} to establish that $\{\whh(t;\theta,\cdot)\}_\theta$ is dense in \aa{$L_R^\infty(\phh)$} we can rewrite the above condition as
\begin{equs}
 \int \bar g(x) \Hhh'(\theta; x,\ell-1)  \nu_{0}(\d  x)
 &=\int \bar g(  x) \int f(\theta') \sigma_h(\Hhh'(\theta'; x,\ell)
 ) \Phh(\d \theta') \nu_0(\d x)\\
&=\int f(\theta') \int  \bar g(x) \sigma_h(\Hhh'(\theta'; x,\ell)
) \nu_{0}(\d  x)\,\Phh(\d \theta')=0
\end{equs}
for all \aa{$f \in L_R^\infty(\phh)$}, where 
in the last line we have applied Fubini's theorem. This is true only if
\begin{equ}\label{e:fubinii}
\int \bar g( x) \sigma_h(\Hhh'(\theta'; x,\ell)
) \nu_{0}(\d  x)=0\qquad \text{for $\phh$-almost all } \theta'\in \Ohh\,.
\end{equ}
which, by the induction assumption, is only true if $\bar g(x) \equiv 0$, showing \eref{e:toshow} and therefore the claim.

\end{proof}

\subsection{Preservation of expressivity during training}

Recalling the definition of $L_R^\infty(\phh) = \{f \in L^2(\phh)~:~\sup_\Theta |f| \leq R\} $ we have
\begin{lemma}[Bidirectional diversity, Step 1 in \cite{pham20global}, proof of Thm.~46]\label{l:support}
Let $\whh(t;\cdot,\cdot), \wxh(t;\cdot)$ be the mean-field parameter functions solving \eref{e:mfpde} with initial condition $\whh^0(\cdot,\cdot), \wxh^0(\cdot)$. If \aref{a:1} holds, then at any time $t>0$ we have that
\begin{equs}
\supp(\wxh(t;\theta),\whh(t;\cdot,\theta),\whh(t;\theta,\cdot)~:~\theta \in \Theta) &= \Pxh \times \aa{L_R^\infty(\phh)\times L_R^\infty(\phh)}
\end{equs}
\end{lemma}

To prove the bidirectional diversity result we will consider the flow induced by \eref{e:mfpde} on any value of the (parametric) initial condition. From now on we denote by $\dtp{f,g}$ the inner product in $L^2(\phh)$.

\begin{proof}[Proof of \lref{l:support}] Consider a MF trajectory $W(t)$ and a triple $u = (u_1,u_2,u_3) \in \mathbb R^d\times L_R^\infty(\phh)\times L_R^\infty(\phh)$, representing respectively values of $(\wxh(\theta), \whh(\cdot, \theta), \whh(\theta, \cdot))$. To characterize the evolution of a triple $u$ we consider the flow
\begin{equs}
  \frac{\partial}{\partial t} \axh^+ (t; u) &= - \beta(t) \int \Delta F(W(t), \xx)\\&\qquad\sum_{i = 1}^{L+1}\dtp{ \Gamma_{i-1}(W(t), \cdot, \xx), \ahf^+ (t, \cdot; u)} \sigma_h'(\HH[W](\xx,\ahs^+ (t,\cdot ; u),\axh^+ (t ; u),i)) \xx_{i}\, \nu(\d \xx)\\
  \frac{\partial}{\partial t} \ahf^+ (t, \theta; u) &= - \beta(t) \chi_R(\ahf^+ (t, \theta; u)) \int \Delta F(W(t),\xx) \\&\qquad\sum_{i = 1}^{L+1} \dtp{ \Gamma_{i-1}(W(t),\cdot , \xx), \ahf^+ (t, \cdot; u)} \sigma_h(\HH[W](\xx,\ahs^+ (t,\cdot ; u),\axh^+ (t ; u),i)) \nu(\d \xx)\\
  \frac{\partial}{\partial t} \ahs^+ (t, \theta'; u) &= - \beta(t) \chi_R(\ahs^+ (t, \theta'; u))\int \Delta F(W(t),\xx)\\&\qquad \sum_{i = 1}^{L+1} \dtp{ \Gamma_{i-1}(W(t),\cdot , \xx), \ahf^+ (t, \cdot; u)} \sigma_h(\HH[W](\theta',\xx,W,i+1)) \nu(\d \xx)\\
  \end{equs}
  \normalsize
with initial conditions $\axh(0; u)^+ = u_1$, $\ahs(0,\cdot; u)^+ = u_2$, $\ahf(0,\cdot; u)^+ = u_3$, where $\Gamma_i(W,\theta, \xx), \Delta F(W,\xx)$ were defined in Appendix~\ref{s:RHS} and
\begin{equ}
  \HH[W](\xx,\ahs^+ (t,\cdot ; u),\axh^+ (t ; u),i) := \dtp{\ahs^+ (t,\cdot ; u)\,,\,\sigma_h(\HH[W](\cdot, \xx, i+1))} + \axh^+ (t ; u) \xx_{-i}
\end{equ}
These flows track the evolution of mean-field parameters in the space where their evolution is naturally embedded: we see that the MF trajectory solving  \eref{e:mfpde} satisfies
\begin{equs}
  \wxh(t,\theta) &= \axh^+(t;\wxh(0;\theta), \whh(0;\theta, \cdot),\whh(0;\cdot, \theta))\\ \whh(t,\theta, \cdot ) &= \ahf^+(t;\wxh(0;\theta), \whh(0;\theta, \cdot),\whh(0;\cdot, \theta))\label{e:aflow}\\
  \whh(t,\cdot, \theta ) &= \ahs^+(t;\wxh(0;\theta), \whh(0;\theta, \cdot),\whh(0;\cdot, \theta))
\end{equs}

We proceed construct, for all finite $T>0$ and every $u^+ = (u_1^+,u_2^+,u_3^+) \in \mathbb R^d\times L_R^\infty(\phh)\times L_R^\infty(\phh)$ an initial condition $u^- = (u_1^-,u_2^-,u_3^-) \in \mathbb R^d\times L_R^\infty(\phh)\times L_R^\infty(\phh)$ that reaches $u^+$ after time $T$, \ie such that
\begin{equ}\label{e:endpoint}
  \axh^+(T;u^-) = u_1^+ \qquad \ahf^+(T, \cdot ;u^-) = u_2^+ \qquad \ahs^+(T, \cdot ;u^-) = u_3^+
\end{equ}
To do so we consider the reverse-time dynamics on the interval $(0,T)$, described by the flow
\small
\begin{equs}
  \frac{\partial}{\partial t} \axh^- (t; u)
  &= - \beta(T-t) \int \Delta F(W(T-t), \xx)
  \\
  & \qquad \sum_{i = 1}^{L+1}\dtp{ \Gamma_{i-1}(W(T-t), \cdot, \xx), \ahf^- (t, \cdot; u)} \sigma_h'(\HH(\xx,\ahs^- (t,\cdot ; u),\axh^- (t ; u),i)) \xx_{i}\, \nu(\d \xx)
  \\
  \frac{\partial}{\partial t} \ahf^- (t, \theta; u)
  &= - \beta(T-t) \chi_R(\ahf^- (t, \theta; u)) \int \Delta F(W(T-t),\xx)
  \\
  & \qquad\sum_{i = 1}^{L+1} \dtp{ \Gamma_{i-1}(W(T-t),\cdot , \xx), \ahf^- (t, \cdot; u)} \sigma_h(\HH(\xx,\ahs^- (t,\cdot ; u),\axh^- (t ; u),i)) \nu(\d \xx)
  \\
  \frac{\partial}{\partial t} \ahs^- (t, \theta'; u)
  &= - \beta(T-t) \chi_R(\ahs^- (t, \theta'; u)) \int \Delta F(W(T-t),\xx)
  \\
  &\qquad \sum_{i = 1}^{L+1} \dtp{ \Gamma_{i-1}(W(T-t),\cdot , \xx), \ahf^- (t, \cdot; u)} \sigma_h(\HH(\theta',\xx,W,i)) \nu(\d \xx)\\
  \end{equs}
  \normalsize
initialized at $\axh^- (0; u) = u_1$, $\ahf^- (0, \cdot; u) = u_2$ and $\ahs^- (0, \cdot; u) = u_3$. Note that, by construction, $\tilde \ahf^-(t) = \ahf^- (T-t, \theta; u) $, $\tilde \ahs^-(t) = \ahs^- (T-t, \theta; u) $ and $\tilde \axh^-(t) = \axh^- (t; u)$ solve the same equation as $\axh^+(t; u)$, $\ahf^+(t,\cdot; u)$, $\ahf^+(t,\cdot; u)$ with initial condition $\tilde \ahf^-(0, \cdot ) = \ahf^-(T,\cdot; u^+)$, $\tilde \ahs^-(0, \cdot ) = \ahs^-(T,\cdot; u^+)$, $\tilde \axh^-(0 ) = \axh^-(T; u^+)$.
By existence and uniqueness of the solution of this system of ODEs both forward and backward in time proven in \sref{s:existenceuniqueness}, we must have that, setting $u^- = (u_1^-,u_2^-, u_3^-) := (\axh^-(T,\cdot; u^+),\ahf^-(T,\cdot ; u^+), \ahs^-(T,\cdot; u^+))$ as the initial condition of
 \eref{e:aflow}, the endpoint of the trajectory of satisfies \eref{e:endpoint} as desired. \aa{Finally, we show that the point $u^-$ is in $\mathbb R \times L_R^\infty(\phh)\times L_R^\infty(\phh)$. This follows immediately upon showing that the set $\mathbb R \times L_R^\infty(\phh)\times L_R^\infty(\phh)$ is invariant with respect to the flow maps $(\axh^+,\ahf^+, \ahs^+)$, $(\axh^-,\ahf^-, \ahs^-)$ induced by the ODEs. The forward invariance of $\mathbb R$ for $\wxh$ under both forward and backward flow maps follows from the Lipschitz bounds on the RHS of the corresponding ODEs, established in \sref{s:existenceuniqueness}. It remains to prove forward invariance of $L_R^\infty(\phh)$, which we now do by contradiction. Assuming that $L_R^\infty(\phh)$ is not invariant with respect to $(\axh^+,\ahf^+, \ahs^+)$, $(\axh^-,\ahf^-, \ahs^-)$, then by the continuity of the flow maps, there must exist $\theta, \theta' \in \supp\, \phh$ with $|\whh(t;\theta, \theta')| = K$ such that $\partial_t |\whh(t;\theta, \theta')| >0$, which is impossible given that $\partial_t |\whh(t;\theta, \theta')| = 0$, since $\chi_R(\whh(\theta, \theta')) = 0$, for all such $\theta, \theta'$ .  }

By continuity of the solution map $u \mapsto (\axh^+(T; u), \ahf^+(T, \cdot; u), \ahs^+(T, \cdot; u))$, for any $\epsilon > 0$ there exists a neighborhood $U$ of $u^- \in  \mathbb R \times L_R^\infty(\phh) \times L_R^\infty(\phh)$ such that
\begin{equ}
  \|(\axh^+(T;u), \ahf^+(T,\cdot;u), \ahs^+(T,\cdot;u)) - u^+\| < \epsilon
\end{equ}
for all $u \in U$.
This finally implies, by  \aref{a:1}c), that $(\wxh(T;\theta), \whh(T; \theta, \cdot), \whh(T; \cdot, \theta))$ has full support in $\mathbb R^d\times L_R^\infty(\phh)\times L_R^\infty(\phh)$, which in turn proves the claim.

\end{proof}

\subsection{Proof of \tref{t:optimality}}

The proof of  \tref{t:optimality} is carried out by adapting the argument from  Theorem 50 in \cite{pham20global}, to the present setting. We recall that, writing $\Delta F[W](\xx) := \hat F( \xx;W(t)) - F^*( \xx)$ and using the definition \eref{e:hh} we have
\begin{equ}
  \partial_t  \why(t;\theta)  = - \int  \Delta F[W](\xx)  
  \sigma_h\big(\HH(\theta; \xx, 0)\big) \nu(\d \xx)
\end{equ}
so that, by the convergence assumption, we have that for every $\epsilon>0$ there exists a $T>0$ such that for almost every $\theta \in \supp(\phh)$
\begin{equ}
   | \int  \Delta F[W](\xx)  
  \sigma_h\big(\HH(\theta; \xx, 0)\big) \nu(\d  \xx)| \leq \epsilon\,.
\end{equ}
We proceed to prove that $\Delta F[W]$ converges in $L^2(\nu)$ to $\Delta F[\bar W]$ as $t\to \infty$. To do so we define
\begin{equ}
  \delta_i(t,\xx,\theta) = \pd{\sigma_h(\HH[\bar W](\theta; \xx, i)) - \sigma_h(\HH[W(t)](\theta; \xx, i))}
\end{equ}
for which by boundedness and Lipschitz continuity of $\sigma_h$ we have
\begin{equs}
  \delta_L(t, \xx, \theta ) &\leq K |\bar W_{xh}(\theta)\xx_{-L} - W_{xh}(t;\theta)\xx_{-L}|\\
  \delta_i(t,\xx,\theta)& \leq K \left(|\bar W_{xh}(\theta)\xx_{-L} -  W_{xh}(t;\theta)\xx_{-L}| + K \int |\bar W_{hh}(\theta, \theta') -  W_{hh}(t;\theta,\theta' )|\phh(\d \theta' ) \right. \\ & \qquad \qquad \qquad \qquad  \left.+ \int |\bar W_{hh}(\theta, \theta')\delta_{i+1}(t,\xx,\theta')| \phh(\d \theta')\right)
\end{equs}
Therefore, denoting by $\d \boldsymbol \theta$ the differential $\d \theta^{(0)},\dots, \d \theta^{(L)}$ we have that
\begin{equs}
  \int |&\Delta F[\bar W](\xx) - \Delta F[W(t)](\xx)|^2 \nu (\d \xx) \\
  &= \int |\hat F(\bar W; \xx) - \hat F(W(t);\xx)|^2\nu (\d \xx)\\
  & \leq \int \pc{K \int |\bar W_{hy}(\theta) - W_{hy}(t;\theta)| \phh(\d\theta) + \int \bar W_{hy}(\theta) \delta_0(t, \xx, \theta) \phh(\d \theta)}^2 \nu (\d \xx)\\
  & \leq K^{2L} \sum_{i=0}^L \int |\bar W_{hy}(\theta^{(0)})|^2 \pc{\prod_{j=1}^{i-1} |\bar W_{hh}(\theta^{(j-1)},\theta^{(j)})|}^2\pd{\bar W_{hh}(\theta^{(i-1)},\theta^{(i)})- W_{hh}(t;\theta^{(i-1)},\theta^{(i)})}^2\phh^{\otimes L+1}(\d \boldsymbol \theta)\\&
  \quad +  K^{2L} \sum_{i=0}^L \int |\bar W_{hy}(\theta^{(0)})|^2 \pc{\prod_{j=1}^{i-1} |\bar W_{hh}(\theta^{(j-1)},\theta^{(j)})|}^2\pd{\bar W_{xh}(\theta^{(i-1)})- W_{xh}(t;\theta^{(i-1)})}^2\phh^{\otimes L+1}(d \boldsymbol \theta) \mathbb E_X[\|\xx\|^2]\\
  & \quad + K^2 \int |\bar W_{hy}(\theta) - W_{hy}(t;\theta)|^2 \phh(d\theta)\label{e:lastbound}
\end{equs}
and by \aref{a:1} we have that the above goes to $0$ as $t \to \infty$.

Having proven the convergence of $\Delta F[W(t)]$ to  $\Delta F[\bar W]$ we proceed to prove the claim of the theorem. By boundedness of $\sigma_h$ we have that for every $\theta \in \supp(\phh)$
\begin{equs}
  | \int  \Delta F[\bar W]  
 \sigma_h & \big(\HH(\theta; \xx, 0)\big) \nu(\d  \xx)|\\&
 \leq K | \int  (\Delta F[\bar W](\xx)  - \Delta F[ W](\xx)) \nu(\d  \xx)| + | \int  \Delta F[ W](\xx)  
\sigma_h & \big(\HH(\theta; \xx, 0)\big) \nu(\d  \xx)|\\&
\leq K \epsilon
\end{equs}
By continuity of $\sigma_h$  we have that for every $\epsilon > 0$
\begin{equ}
  |\int  \Delta F[\bar W](\xx)  
 f(\xx) \nu(\d \xx)| \leq K \epsilon
\end{equ}
uniformly over $f(\xx) \in S$ where $S = \{\sigma_h(\HH(\theta; \xx, 0))~:~\theta \in \phh\}$, implying that
$
  |\int  \Delta F[\bar W](\xx)  
 f(\xx) \nu(\d \xx)|= 0$ for all  $f \in S$.
Since from \pref{p:spanning} we have that $\span(\sigma_h(\HH(\theta; \xx, 0))) = L^2(\nu)$, the above result immediately yields that for $\nu$-almost every $\xx$, $\Delta F[\bar W](\xx) = 0$, so that $\mathcal L(\bar W) = 0$.

Finally, we prove the desired result by connecting $\mathcal L(\bar W)$ and $\mathcal L(W(t))$:
\begin{equs}
|\mathcal L(\bar W) - \mathcal L(W(t))|& = |\int \Delta F[\bar W](\xx)^2 - \Delta F[W(t)](\xx)^2 \nu(\d \xx)|
& \leq 2K \|\hat F(\bar W;\cdot ) - \hat F(W(t);\cdot)\|_\nu \,,
\end{equs}
which by \eref{e:lastbound} goes to $0$ with $t \to \infty$. Combining the above we have
\begin{equ}
  \lim_{t \to \infty} \mathcal L(W(t)) \leq \mathcal L(\bar W) + \lim_{t \to \infty} |\mathcal L(\bar W) - \mathcal L(W(t))| = 0
\end{equ}
which proves the claim.\qed\\[10pt]

Author affiliations:

\end{document}